%% file: Online_Ranking_with_Restricted_Feedback.tex
\documentclass[twoside,11pt]{article}

\usepackage{jmlr2e}

\usepackage{hyperref}
\usepackage{url}
\usepackage{subfigure} 
\usepackage[bottom]{footmisc}

\usepackage{algorithm}
\usepackage{algorithmic}
\usepackage{amsmath}[11pt]
\usepackage{amsthm}
\usepackage{amsfonts}[11pt]
\usepackage{mathtools}
\usepackage{bbm}

\newtheorem{thm}{Theorem}
\newtheorem{cor}[thm]{Corollary}
\newtheorem{lem}[thm]{Lemma}
\newcommand{\bc}{\begin{center}}
\newcommand{\ec}{\end{center}}
\newcommand{\bq}{\begin{quote}}
\newcommand{\eq}{\end{quote}}

\newcommand{\be}{\begin{equation}}
\newcommand{\ee}{\end{equation}}
\newcommand{\beqa}{\begin{eqnarray*}}
\newcommand{\eeqa}{\end{eqnarray*}}
\newcommand{\beqn}{\begin{eqnarray}}
\newcommand{\eeqn}{\end{eqnarray}}
\newcommand{\bbibl}{\begin{thebibliography}{99}}
\newcommand{\ebibl}{\end{thebibliography}}

\newcommand{\ba}{\begin{array}}
\newcommand{\ea}{\end{array}}

\DeclareMathOperator*{\argsort}{argsort}
\DeclareMathOperator*{\argmax}{argmax}
\DeclareMathOperator*{\argmin}{argmin}

\newcommand\reals{\mathbb{R}}
\newcommand{\E}{\mathbb{E}}
\newcommand*{\permcomb}[4][0mu]{{{}^{#3}\mkern#1#2_{#4}}}
\newcommand*{\perm}[1][-3mu]{\permcomb[#1]{P}}
\newcommand\listnet{\phi_{\mathrm{LN}}}
\newcommand\sdcg{\phi_{\mathrm{SD}}}
\DeclarePairedDelimiter{\ceil}{\lceil}{\rceil}

\ShortHeadings{Online Learning to Rank}{Chaudhuri and Tewari}
\firstpageno{1}

\begin{document}

\title{Online Learning to Rank with Top-k Feedback}

\author{\name Sougata Chaudhuri $^1$  \email sougata@umich.edu \\
             \name Ambuj Tewari $^{1, 2}$ \email tewria@umich.edu\\
             \addr Department of Statistics $^1$ \\
              Department of Electrical Engineering and Computer Science $^2$ \\
              University of Michigan\\
              Ann Arbor, MI 48109, USA}

\editor{}
\maketitle

\input{abstract}

\begin{keywords}
  Learning to Rank, Online Learning, Partial Monitoring, Online Bandits, Learning Theory
\end{keywords}

\input{Introduction}

\input{OnlineRankingNC}

\input{OnlineRankingC}

\input{Experiments}

\input{QuestionsandConclusion}

\acks{The authors acknowledge the support of NSF under grants IIS-1319810 and CAREER IIS-1452099. }

\bibliography{OnlineRankingBibliography}


\newpage

\appendix
\section*{Appendix A.}
\label{Appendix-A}
We provide technical details of results of Online Ranking with Restricted Feedback- Non Contextual Setting.

\input{Appendix-A}

\section*{Appendix B.}
\label{Appendix-B}
We provide technical details of results of Online Ranking with Restricted Feedback- Contextual Setting.

\input{Appendix-B}

\end{document}

%% file: abstract.tex
\begin{abstract}
We consider two settings of online learning to rank where feedback is restricted to top ranked items. The problem is cast as an online game between a learner and sequence of users, over $T$ rounds. In both settings, the learners objective is to present ranked list of items to the users. The learner's performance is judged on the entire ranked list and true relevances of the items. However, the learner receives highly restricted feedback at end of each round, in form of relevances of only the top $k$ ranked items, where $k \ll m$. The first setting is \emph{non-contextual}, where the list of items to be ranked is fixed. The second setting is \emph{contextual}, where lists of items vary, in form of traditional query-document lists. No stochastic assumption is made on the generation process of relevances of items and contexts. We provide efficient ranking strategies for both the settings. The strategies achieve $O(T^{2/3})$ regret, where regret is based on popular ranking measures in first setting and ranking surrogates in second setting. We also provide impossibility results for certain ranking measures and a certain class of surrogates, when feedback is restricted to the top ranked item, i.e. $k=1$. We empirically demonstrate the performance of our algorithms on simulated and real world datasets.\\

\end{abstract}

%% file: Introduction.tex
\section{Introduction}
\label{introduction}
Learning to rank \citep{liu2011learning} is a supervised machine learning problem, where the output space consists of \emph{rankings} of objects. Most learning to rank methods are based on supervised \emph{batch} learning, i.e., rankers are trained on batch data in an offline setting. The accuracy of a ranked list, in comparison to the actual relevance of the documents, is measured by various ranking measures, such as Discounted Cumulative Gain (DCG) \citep{jarvelin2000}, Average Precision (AP) \citep{baeza1999} and others.

Collecting reliable training data can be expensive and time consuming. In certain applications, such as deploying a new web app or developing a custom search engine, collecting large amount of high quality labeled data might be infeasible \citep{sanderson2010test}. Moreover, a ranker trained from batch data might not be able to satisfy rapidly changing user needs and preferences. Thus, a promising direction of research is development of online ranking systems, where a ranker is updated on the fly. One type of online ranking models learn from implicit feedback inferred from user clicks on ranked lists \citep{hofmann2013balancing}. However, there are some potential drawbacks in learning from user clicks. It is possible that the system is designed for explicit ratings but not clicks. Moreover, a clicked item might not actually be relevant to the user and there is also the problem of bias towards top ranked items in inferring feedback from user clicks \citep{joachims2002}. 

We develop models for online learning of ranking systems, from explicit but \emph{highly restricted} feedback. At a high level, we consider a ranking system which interacts with users over a time horizon, in a sequential manner. At each round, the system presents a ranked list of $m$ items to the user, with the quality of the ranked list judged by the relevance of the items to the user. The relevance of the items, reflecting varying user preferences, is encoded as relevance vectors. The system's objective is to learn from the feedback it receives and update its ranker over time, to satisfy as many users as possible. However, the feedback that the system receives at end of each round is not the full relevance vector, but relevance of only the top $k$ ranked items, where $k\ll m$ (typically, $k=$ $1$ or $2$). We consider two problem settings under the general framework: \emph{non-contextual} and \emph{contextual}. In the first setting, we assume that the set of items to be ranked are fixed (i.e., there are no context on items), with the relevance vectors varying according to users' preferences. In the second setting, we assume that set of items vary, as traditional query-document lists. We highlight two motivating examples for such feedback model, encompassing \emph{economic and user-burden} constraints and \emph{privacy} concerns.

{\bf Privacy Concerns}: Assume that a medical company wants to build an app to suggest activities (take a walk, meditate, watch relaxing videos, etc.) that can lead to reduction of stress in a certain highly stressed segment of the population. The activities do not have contextual representation and are fixed over time.  Not all activities are likely to be equally suitable for everyone under all conditions since the effects of the activities vary depending on the user attributes like age \& gender and on the context such as time of day \& day of week. The user has liberty to browse through all the suggested activities, and the company would like the user to rate every activity (may be on an $1-5$ scale), reflecting the relevances, so that it can keep refining its ranking strategy. However, in practice, though the user may scan through all suggested activities and have a rough idea about how relevant each one is to her; she is unlikely to give feedback on the usefulness (relevance) of every activity due to privacy concerns and cognitive burden. Hence, in exchange of the user using the app, the company only asks for careful rating of the top $1$ or $2$ ranked activities. The apps performance would still be based on the full ranked list, compared to the implicit relevance vector that the user generates, but it gets feedback on the relevances of only top $1$ or $2$ ranked activities.


{\bf Economic Constraints}: Assume that a small retail company wants to build an app that produces a ranked list of suggestions to a user query, for categories of different products. The app develops features representing the categories, and thus, the interaction with the users happen in a traditional query-documents lists framework (user query and retrieved activities are jointly represented through a feature matrix). Different categories are likely to have varying relevance to different users, depending on user characteristics such as age, gender, etc. Like in the first example, the user has liberty to browse through all the suggestions but she will likely feel too burdened to give carefully considered rating on each suggestion, unless the company provides some economic incentives to do so. Though the company needs high quality feedback on each suggestion to keep refining the ranking strategy, it cannot afford to give incentives due to budget constraints. Similar to the first example, the company only asks, and possibly pays, for rating on the top $1$ or $2$ ranked suggestions, in lieu of using the app, but its performance is judged on the full ranked list and implicit relevance vector.

We cast the online learning to rank problem as an online game between a learner and an adversary, played over time horizon $T$. That is, we do not make any \emph{stochastic} assumption on the relevance vector generation process or the context (features) generation process (in the second problem setting). The adversary is considered to be \emph{oblivious} to the learner's strategies. We separately discuss the two problem settings, and our contributions in each, in greater details.

{\bf Non-contextual setting}: Existing work loosely related to ranking of a fixed set of items to satisfy diverse user preferences \citep{radlinski2008,radlinski2009,agrawal2009,wen2014efficient} has focused on learning an optimal ranking of a subset of items, to be presented to an user, with performance judged by a simple $0$-$1$ loss. The loss in a round is $0$  if among the top $k$ (out of $m$) items presented to a user, the user finds at least one relevant item. All of the work falls under the framework of \emph{online bandit} learning. In contrast, our model focuses on optimal ranking of the entire list of items, where the performance of the system is judged by practical ranking measures like DCG and AP. The challenge is to decide when and how efficient learning is possible with the highly restricted feedback model. Theoretically, the top $k$ feedback model is neither full-feedback nor bandit-feedback since not even the loss (quantified by some ranking measure) at each round is revealed to the learner. The appropriate framework to study the problem is that of \emph{partial monitoring} \citep{cesa2006}.  A very recent paper shows another practical application of partial monitoring in the stochastic setting \citep{lincombinatorial2014}. Recent advances in the classification of partial monitoring games tell us that the minimax regret, in an adversarial setting, is governed by a property of the loss and feedback functions called \emph{observability} \citep{bartok2013, foster2011}, where observability is of two kinds: \emph{local} and \emph{global}. 

{\bf Our contributions}: We instantiate these general observability notions for our problem with top $1$ ($k=1$) feedback. We prove that, for some ranking measures, namely PairwiseLoss \citep{duchi2010}, DCG and Precision@$n$ \citep{liu2007}, global observability holds. This immediately shows that the upper bound on regret scales as $O(T^{2/3})$. Specifically for PairwiseLoss and DCG, we further prove that local observability fails, when restricted to the top $1$ feedback case, illustrating that their \emph{minimax} regret scales as $\Theta(T^{2/3})$. However, the generic algorithm that enjoys $O(T^{2/3})$ regret for globally observable games necessarily maintains explicit weights on each action in learner's action set. It is impractical in our case to do so, since the learner's action set is the exponentially large set of $m!$ rankings over $m$ objects. We propose a generic algorithm for learning with top $k$ feedback, which uses blocking and a black-box full information algorithm. Specifically, we instantiate the black box algorithm with Follow The Perturbed Leader (FTPL) strategy, which leads to an efficient algorithm achieving $O(T^{2/3})$ regret bound for PairwiseLoss, DCG and Precision@$n$, with $O(m \log m)$ time spent per step. Moreover, the regret of our efficient algorithm has a logarithmic dependence on number of learner's actions (i.e., polynomial dependence on $m$), whereas the generic algorithm has a linear dependence on number of actions (i.e., exponential dependence on $m$). 

For several measures, their \emph{normalized} versions are also considered. For example, the normalized versions of PairwiseLoss, DCG and Precision@$n$ are called AUC \citep{cortes2004}, NDCG \citep{jarvelin2002} and AP respectively. We show an unexpected result for the normalized versions: \emph{they do not} admit sub-linear regret algorithms with top $1$ feedback. This is despite the fact that the opposite is true for their unnormalized counterparts! Intuitively, the normalization makes it hard to construct an unbiased estimators of the (unobserved) relevance vectors. Surprisingly, we are able to translate this intuitive hurdle into a provable impossibility. 

We also present some preliminary experiments on simulated datasets to explore the performance of our efficient algorithm and compare its regret to its full information counterpart.

{\bf Contextual Setting}: The requirement of having a fixed set of items to rank, in the first part of our work, somewhat limits practical applicability. In fact, in the classic multi-armed bandit problem, while non-contextual bandits have received a lot of attention, the authors  \cite{langford2008epoch} mention that ``settings with no context information are rare in practice". The second part of our work introduces context, by combining query-level ranking with the explicit but restricted feedback model. At each round, the adversary generates a document list of length $m$, pertaining to a query. The learner sees the list and produces a real valued score vector to rank the documents. We assume that the ranking is generated by sorting the score vector in descending order of its entries. The adversary then generates a relevance vector but, like in the non-contextual setting, the learner gets to see the relevance of only the top $k$ items of the ranked list. The learner's loss in each round, based on the learner's score vector and the \emph{full} relevance vector, is measured by some continuous ranking surrogates. We focus on continuous surrogates, e.g., the cross entropy surrogate in ListNet \citep{cao2007learning} and hinge surrogate in RankSVM \citep{joachims2002}, instead of discontinuous ranking measures like DCG, or AP, because the latter lead to intractable optimization problems in the query-documents setting. Just like in the non-contextual setting, we note that the top $k$ feedback model is neither full feedback nor bandit feedback models. The problem is an instance of partial monitoring, \emph{extended to a setting with side information} (documents list) and an \emph{infinite set of learner's moves} (all real valued score vectors). For such an extension of partial monitoring there exists no generic theoretical or algorithmic framework to the best of our knowledge. 

{\bf Our contributions}: In this setting, first, we propose a general, efficient algorithm for online learning to rank with top $k$ feedback and show that it works in conjunction with a number of ranking surrogates. We characterize the minimum feedback required, i.e., the value of $k$, for the algorithm to work with a particular surrogate by formally relating the feedback mechanism with the structure of the surrogates. We then apply our general techniques to three convex ranking surrogates and one non-convex surrogate. The convex surrogates considered are from three major learning to ranking methods: squared loss from a \emph{pointwise} method \citep{cossock2008statistical}, hinge loss used in the \emph{pairwise} RankSVM \citep{joachims2002} method, and (modified) cross-entropy surrogate used in the \emph{listwise} ListNet \citep{cao2007learning} method. The non-convex surrogate considered is the SmoothDCG surrogate \citep{chapelle2010gradient}. For the three convex surrogates, we establish an $O(T^{2/3})$ regret bound. 

The convex surrogates we mentioned above are widely used but are known to fail to be calibrated with respect to NDCG \citep{ravikumar2011ndcg}. Our second contribution is to show that for the entire class of NDCG calibrated surrogates, no online algorithm can have sub-linear (in $T$) regret with top $1$ feedback, i.e., the minimax regret of an online game for any NDCG calibrated surrogate is $\Omega(T)$. The proof for this result relies on exploiting a connection between the construction of optimal adversary strategies for hopeless \emph{finite action} partial monitoring games \citep{piccolboni2001discrete} and the structure of NDCG calibrated surrogates. We only focus on NDCG calibrated surrogates for the \emph{impossibility} results since no (convex) surrogate can be calibrated for AP and ERR \citep{calauzenes2012non}. This impossibility result is the first of its kind for a natural partial monitoring problem with side information when the learner's action space is infinite. Note, however, that there does exist work on partial monitoring problems with continuous learner actions, but without side information \citep{kleinberg2003value,cesa2006}, and vice versa \citep{bartok2012partial,gentile2014multilabel}.

We apply our algorithms on benchmark ranking datasets, demonstrating the ability to efficiently learn a ranking function in an online fashion, from highly restricted feedback.

The rest of the paper is divided into the following sections. Section~\ref{NC} and its sub-sections detail the notations, definitions and technicalities associated with online ranking with restricted feedback in the non-contextual setting. Section~\ref{C} and its subsections detail the notations, definitions and technicalities associated with online ranking with restricted feedback in the contextual setting. Section~\ref{experiments} demonstrates the performance of our algorithms on simulated and commercial datasets. Section~\ref{conclusion} discusses open questions and future directions of research.

%% file: OnlineRankingNC.tex
\section{Online Ranking with Restricted Feedback- Non Contextual Setting}
\label{NC}
All proofs not in main text are in Appendix A.

\subsection{Notation and Preliminaries}
\label{preliminaries-NC}
The fixed $m$ items to be ranked are numbered $\{1,2,\ldots,m\}$. A permutation $\sigma$ gives a mapping from ranks to items and its inverse $\sigma^{-1}$ gives a mapping from items to ranks. Thus, $\sigma^{-1}(i)=j$ means item $i$ is placed at position $j$ while $\sigma(i)=j$ means item $j$ is placed at position $i$. The supervision is in form of a relevance vector $R= \{0,1,\ldots,n\}^m$, representing relevance of each document to the query. If $n=1$, the relevance vector is binary graded. For $n>1$, relevance vector is multi-graded. $R(i)$ denotes $i$th component of $R$. \emph{The subscript $t$ is exclusively used to denote time $t$}.  We denote $\{1,\ldots,n\}$ by $[n]$. The learner can choose from $m!$ actions (permutations) whereas nature/adversary can choose from $2^m$ outcomes (when relevance levels are restricted to binary) or from $n^m$ outcomes (when there are $n$ relevance levels, $n>2$). We sometimes refer to the learner's $i$th action (in some fixed ordering of $m!$ available actions) as $\sigma_i$ (resp. adversary's $i$th action as $R_i$). Note that $\sigma_i^{-1}$ simply means that we are viewing permutation $\sigma_i$ as mapping from items to ranks.  Also, a vector can be row or column vector depending on context.

At round $t$, the learner outputs a permutation (ranking) $\sigma_t$ of the objects (possibly using some internal randomization, based on feedback history so far), and simultaneously, adversary generates relevance vector $R_t$. The quality of $\sigma_t$ is judged against $R_t$ by some ranking measure $RL$. \emph{Crucially, only the relevance of the top ranked object, i.e., $R_t(\sigma_t(1))$, is revealed to the learner at end of round $t$}. Thus, the learner gets to know neither $R_t$ (full information problem) nor $RL(\sigma_t,R_t)$ (bandit problem). The objective of the learner is to minimize the expected regret with respect to best permutation in hindsight:
\begin{equation}
\begin{aligned}
\label{eq:objectiveregret}
\E_{\sigma_1, \ldots, \sigma_T}\left[ \sum_{t=1}^T RL(\sigma_t, R_t) \right] - \underset{\sigma}{\min} \ \sum_{t=1}^T RL(\sigma, R_t) .
\end{aligned}
\end{equation}
When $RL$ is a gain, not loss, we need to negate the quantity above.
The worst-case regret of a learner strategy is its maximal regret over all possible choices of $R_1,\ldots,R_T$. The {\bf minimax regret} is the minimal worst-case regret over all learner strategies.


\subsection{Ranking Measures}
\label{rankingmeasures-NC}

We consider ranking measures which can be expressed in the form $f(\sigma) \cdot R$, where the function $f:\mathbb{R}^m \rightarrow \mathbb{R}^m$ is composed of $m$ copies of a univariate, monotonic, scalar valued function. Thus, $f(\sigma)= [f^s(\sigma^{-1}(1)), f^s(\sigma^{-1}(2)), \ldots, f^s(\sigma^{-1}(m))]$, where $f^s: \mathbb{R} \rightarrow \mathbb{R}$. Monotonic (increasing) means $f^s(\sigma^{-1}(i))\ge f^s(\sigma^{-1}(j))$, whenever $\sigma^{-1}(i) > \sigma^{-1}(j)$. Monotonic (decreasing) is defined similarly. 
The following popular ranking measures can be expressed in the form $f(\sigma) \cdot r$.

{\bf PairwiseLoss \& SumLoss}: PairwiseLoss is restricted to binary relevance vectors and defined as: 
\[
PL(\sigma,R)= \sum_{i=1}^m \sum_{j=1}^m \mathbbm{1}(\sigma^{-1}(i) < \sigma^{-1}(j))\mathbbm{1}(R(i) < R(j))
\]
PairwiseLoss cannot be directly expressed in the form of $f(\sigma)\cdot R$. Instead, we consider {\bf SumLoss}, defined as: 
\[
SumLoss(\sigma,R)= \sum_{i=1}^m \sigma^{-1}(i)\ R(i)
\]
SumLoss has the form $f(\sigma) \cdot R$, where $f(\sigma)=\sigma^{-1}$. 
It has been shown by \citet{ailon2014} that regret under the two measures are equal:
\begin{equation}
\label{rankloss-sumloss}
\sum_{t=1}^T PL(\sigma_t, R_t) -  \sum_{t=1}^T PL(\sigma, R_t) = \sum_{t=1}^T SumLoss(\sigma_t, R_t) - \sum_{t=1}^T SumLoss(\sigma,R_t) .
\end{equation}

{\bf Discounted Cumulative Gain}: DCG is a gain function which admits non-binary relevance vectors and is defined as: 
\[
DCG(\sigma,R)= \sum_{i=1}^m \frac{2^{R(i)}-1}{\log_2(1+ \sigma^{-1}(i))}
\]
and becomes $\sum_{i=1}^m \frac{R(i)}{\log_2(1+ \sigma^{-1}(i))}$ for $R (i) \in \{0,1\}$. 
Thus, for binary relevance, $DCG(\sigma,R)$  has the form $f(\sigma) \cdot R$, where $f(\sigma)= [\frac{1}{\log_2(1+ \sigma^{-1}(1))}, \frac{1}{\log_2(1+ \sigma^{-1}(2))}, \ldots, \frac{1}{\log_2(1+ \sigma^{-1}(m))}]$.

{\bf Precision@n Gain}: Precision@$n$ is a gain function restricted to binary relevance and is defined as
\[
Precision@n(\sigma,R)= \sum_{i=1}^m \mathbbm{1}(\sigma^{-1}(i) \le n)\ R(i)
\]
Precision@$n$ can be written as $f(\sigma) \cdot R$ where $f(\sigma)=  [\mathbbm{1}(\sigma^{-1}(1)<n), \ldots, \mathbbm{1}(\sigma^{-1}(m)<n)]$. It should be noted that for $n=k$ (i.e., when feedback is on top $n$ items), feedback is actually the same as full information feedback, for which efficient algorithms already exist. 

{\bf Normalized measures are not of the form ${\bf f(\sigma) \cdot R}$}: PairwiseLoss, DCG and Precision@$n$ are unnormalized versions of popular ranking measures, namely, Area Under Curve (AUC), Normalized Discounted Cumulative Gain (NDCG) and Average Precision (AP) respectively.  None of these can be expressed in the form $f(\sigma) \cdot R$.

{\bf NDCG}: NDCG is a gain function, admits non-binary relevance and is defined as:
\[
NDCG(\sigma,R)= \frac{1}{Z(R)}\sum_{i=1}^m \frac{2^{R(i)}-1}{\log_2(1+ \sigma^{-1}(i))}
\]
and becomes $\frac{1}{Z(R)} \sum_{i=1}^m \frac{R(i)}{\log_2(1+ \sigma^{-1}(i))}$ for $R(i) \in \{0,1\}$.  Here $Z(R)= \underset{\sigma}{\max} \sum_{i=1}^m  \frac{2^{R(i)}-1}{\log_2(1+ \sigma^{-1}(i))}$ is the normalizing factor ($Z(R)= \underset{\sigma}{\max} \sum_{i=1}^m  \frac{R(i)}{\log_2(1+ \sigma^{-1}(i))}$ for binary relevance). It can be clearly seen that $NDCG(\sigma, R)= f(\sigma) \cdot g(R)$, where $f(\sigma)$ is same as in DCG but $g(R)= \frac{R}{Z(R)}$ is non-linear in $R$. 

{\bf AP}: AP is a gain function, restricted to binary relevance and is defined as: 
\[
AP(\sigma,R)= \frac{1}{\|R\|_1}\sum\limits_{\substack{i=1}}^m \frac{\sum\limits_{j\le i} \mathbbm{1}(R(\sigma(j))=1)}{i} \mathbbm{1}(R(\sigma(i)=1)
\]
It can be clearly seen that AP cannot be expressed in the form $f(\sigma) \cdot R$. 

{\bf AUC}: AUC is a loss function, restricted to binary relevance and is defined as:
\[
AUC(\sigma,R)= \frac{1}{N(R)} \sum_{i=1}^m \sum_{j=1}^m \mathbbm{1}(\sigma^{-1}(i) < \sigma^{-1}(j))\mathbbm{1}(R(i) < R(j))
\]
where $N(R)=  (\sum_{i=1}^m \mathbbm{1}(R(i)=1)) \cdot (m-\sum_{i=1}^m \mathbbm{1}(R(i)=1))$. It can be clearly seen that AUC cannot be expressed in the form $f(\sigma) \cdot R$. 

{\bf Note}: We will develop our subsequent theory and algorithms for binary valued relevance vectors, and show how they can be extended to multi-graded vectors when ranking measure is DCG/NDCG.

%
%
%
%

\subsection{Relevant Definitions from Partial Monitoring}
\label{partialmonitoring-NC}
We develop all results in context of SumLoss, with binary relevance vector. We then extend the results to other ranking measures. 
Our main results on regret bounds build on some of the theory for abstract partial monitoring games developed by \citet{bartok2013} and \citet{foster2011}.  For ease of understanding, we reproduce the relevant notations and definitions in context of SumLoss. \emph{We will specifically mention when we derive results for top $k$ feedback, with general $k$, and when we restrict to top $1$ feedback}.

{\bf Loss and Feedback Matrices}: The online learning game with the SumLoss measure and top $1$ feedback can be expressed in form of a pair of \emph{loss matrix} and \emph{feedback matrix}. The \emph{loss matrix} $L$ is an $m! \times 2^m$ dimensional matrix, with rows indicating the learner's actions (permutations) and columns representing adversary's actions (relevance vectors). The entry in cell $(i,j)$ of $L$ indicates loss suffered when learner plays action $i$ (i.e., $\sigma_i$) and adversary plays action $j$ (i.e., $R_j$), that is, $L_{i,j}= \sigma_i^{-1} \cdot R_j= \sum_{k=1}^m \sigma^{-1}_i(k)R_j(k)$. The \emph{feedback matrix} $H$ has same dimension as \emph{loss matrix}, with $(i,j)$ entry being the relevance of top ranked object, i.e., $H_{i,j}= R_j(\sigma_i(1))$. When the learner plays action $\sigma_i$ and adversary plays action $R_j$, the true loss is $L_{i,j}$, while the feedback received is $H_{i,j}$. 

Table \ref{lossmatrix-table} and \ref{feedbackmatrix-table} illustrate the matrices, with number of objects $m=3$. In both the tables, the permutations indicate rank of each object and relevance vector indicates relevance of each object. For example, $\sigma_5= 3 1 2$ means object $1$ is ranked $3$, object $2$ is ranked $1$ and object $3$ is ranked 2. $R_5=100$ means object $1$ has relevance level $1$ and other two objects have relevance level $0$. Also, $L_{3,4}= \sigma_3 \cdot R_4= \sum_{i=1}^3 \sigma_3^{-1}(i) R_4(i)= 2\cdot 0 + 1 \cdot 1 + 3 \cdot 1= 4$;  $ H_{3,4}= R_4(\sigma_3(1))= R_4(2)= 1$. Other entries are computed similarly.

\begin{table}[t] 
\caption{Loss matrix $L$ for $m=3$} 
\label{lossmatrix-table}
\begin{center}
\tabcolsep=0.110cm
\begin{tabular}{c c c c c c c c c}  
\hline 
Objects & $R_1$ & $R_2$ & $R_3$ & $R_4$ & $R_5$ & $R_6$ & $R_7$ & $R_8$ \\ [0.4ex] 
\hline 123 & 000 & 001 & 010 & 011 & 100 & 101 & 110 & 111\\ [0.4ex]
\hline 
$\sigma_1= 1 2 3$ & 0 & 3 & 2 & 5 & 1 & 4 & 3 & 6 \\ 
$\sigma_2= 1 3 2$ & 0 & 2 & 3 & 5 & 1 & 3 & 4 & 6 \\ 
$\sigma_3= 2 1 3$ & 0 & 3 & 1 & 4 & 2 & 5 & 3 & 6  \\ 
$\sigma_4= 2 3 1$ & 0 & 1 & 3 & 4 & 2 & 3 & 5 & 6 \\ 
$\sigma_5= 3 1 2$ & 0 & 2 & 1 & 3 & 3 & 5 & 4 & 6 \\
$\sigma_6= 3 2 1$ & 0 & 1 & 2 & 3 & 3 & 4 & 5 & 6 \\
\hline 
\end{tabular} 
\end{center}
\end{table}


\begin{table}[t] 
\caption{Feedback matrix $H$ for $m=3$} 
\label{feedbackmatrix-table}
\begin{center}
\tabcolsep=0.110cm
\begin{tabular}{c c c c c c c c c}  
\hline 
Objects & $R_1$ & $R_2$ & $R_3$ & $R_4$ & $R_5$ & $R_6$ & $R_7$ & $R_8$ \\ [0.4ex] 
\hline 123 & 000 & 001 & 010 & 011 & 100 & 101 & 110 & 111\\ [0.4ex]
\hline 
$\sigma_1= 1 2 3$ & 0 & 0 & 0 & 0 & 1 & 1 & 1 & 1 \\ 
$\sigma_2= 1 3 2$ & 0 & 0 & 0 & 0 & 1 & 1 & 1 & 1 \\ 
$\sigma_3= 2 1 3$ & 0 & 0 & 1 & 1 & 0 & 0 & 1 & 1  \\ 
$\sigma_4= 2 3 1$ & 0 & 1 & 0 & 1 & 0 & 1 & 0 & 1 \\ 
$\sigma_5= 3 1 2$ & 0 & 0 & 1 & 1 & 0 & 0 & 1 & 1 \\
$\sigma_6= 3 2 1$ & 0 & 1 & 0 & 1 & 0 & 1 & 0 & 1 \\
\hline 
\end{tabular} 
\end{center}
\end{table}

Let $\ell_i \in \mathbb{R}^{2^m}$ denote row $i$ of $L$. Let $\Delta$ be the probability simplex in $\mathbb{R}^{2^m}$, i.e., $\Delta= \{ p \in \mathbb{R}^{2^m}: \forall \ 1 \le i \le 2^m, \ p_i \ge 0, \ \sum p_i = 1\}$.  
The following definitions, given for abstract problems by \cite{bartok2013}, has been refined to fit our problem context.

{\bf Definition 1}: Learner action $i$ is called optimal under distribution $p \in \Delta$, if $\ell_i \cdot p \le \ell_ j \cdot p$, for all other learner actions $1 \le j \le m!, \ j \neq i$.  For every action $i \in [m!]$, probability cell of $i$ is defined as $C_i =\{ p \in \Delta: \text{action } i \text{ is optimal under } p\}$. If a non-empty cell $C_i$ is $2^m-1$ dimensional (i.e, elements in $C_i$ are defined by only 1 equality constraint), then associated action $i$ is called \emph{Pareto-optimal}.

Note that since entries in $H$ are relevance levels of objects, there can be maximum of $2$ distinct elements in each row of $H$, i.e., $0$ or $1$ (assuming binary relevance). 

{\bf Definition 2}: The \emph{signal matrix} $S_i$, associated with learner's action $\sigma_i$, is a matrix with 2 rows and $2^m$ columns, with each entry $0$ or $1$, i.e., $S_i \in \{0,1\}^{2 \times 2^m}$. The entries of $\ell$th column of $S_i$ are respectively: $(S_i)_{1,\ell}= \mathbbm{1}(H_{i,\ell}=0)$ and $(S_i)_{2,\ell}= \mathbbm{1}(H_{i,\ell}=1)$. 

Note that by definitions of signal and feedback matrices, the 2nd row of $S_i$ (2nd column of $S^{\top}_i)$) is precisely the $i$th row of $H$. The 1st row of $S_i$ (1st column of $S^{\top}_i)$) is the (boolean) complement of $i$th row of $H$.

\subsection{Minimax Regret for SumLoss}
The minimax regret for SumLoss, restricted to top $1$ feedback, will be established by showing that: a) SumLoss satisfies \emph{global observability}, and b) it does not satisfy \emph{local observability}. 

\subsubsection{Global Observability}
\label{global-NC}
 
{\bf Definition 3}: The condition of \emph{global observability} holds, w.r.t. loss matrix $L$ and feedback matrix $H$, if for every pair of learner's actions $\{\sigma_i,\sigma_j\}$, it is true that $\ell_i -\ell_j \in \oplus_{k \in [m!]} Col(S_k^{\top})$ (where $Col$ refers to column space). 

The global observability condition states that the (vector) loss difference between any pair of learner's actions has to belong to the vector space spanned by columns of (transposed) signal matrices corresponding to all possible learner's actions.
We derive the following theorem on global observability for $SumLoss$.

\begin{thm}
\label{globalobservability}
The global observability condition, as per Definition 3, holds w.r.t. loss matrix  $L$ and feedback matrix $H$ defined for SumLoss, for any $m \ge 1$.
\end{thm}

\begin{proof}
For any $\sigma_a$ (learner's action) and $R_b$ (adversary's action), we have 
\begin{equation*}
\begin{aligned}
L_{a,b}=\sigma_a^{-1} \cdot R_b= \sum_{i=1}^m \sigma_a^{-1}(i)R_b(i) \stackrel{1}{=} \sum_{j=1}^ m j\  R_b(\sigma_a(j)) \stackrel{2}{=}  \sum_{j=1}^m j\ R_b(\tilde{\sigma}_{j(a)} (1)) \stackrel{3}{=} \sum_{j=1}^m j \ (S^{\top}_{\tilde{\sigma}_{j(a)}})_{R_b,2}.\\
\end{aligned}
\end{equation*}
Thus, we have
\begin{equation*}
\begin{aligned}
& \ell_a = [L_{a,1}, L_{a,2},\ldots, L_{a,2^m}] = \\
&  [\sum_{j=1}^m j \ (S^{\top}_{\tilde{\sigma}_{j(a)}})_{R_1,2}, \sum_{j=1}^m j \ (S^{\top}_{\tilde{\sigma}_{j(a)}})_{R_2,2}, .., \sum_{j=1}^m j \ (S^{\top}_{\tilde{\sigma}_{j(a)}})_{R_{2^m},2}] \stackrel{4}{=} \sum_{j=1}^m j \ (S^{\top}_{\tilde{\sigma}_{j(a)}})_{:,2} .
\end{aligned}
\end{equation*}

Equality 4 shows that $\ell_a$ is in the column span of $m$ of the $m!$ possible (transposed) signal matrices, specifically in the span of the 2nd columns of those (transposed) $m$ matrices. 
Hence, for all actions $\sigma_a$, it is holds that $ \ell_a \in { \oplus}_{k \in [m!]} Col(S_k^{\top})$. This implies that $\ell_a - \ell_b \in \oplus_{k \in [m!]} Col(S_k^{\top}), \ \forall \ \sigma_a,\sigma_b. $\\

1. Equality 1 holds because $\sigma_a^{-1}(i)=j \Rightarrow i= \sigma_a(j)$. 

2. Equality 2 holds because of the following reason. For any permutation $\sigma_a$ and for every $j \in [m]$, $\exists$ a permutation $ \tilde{\sigma}_{j(a)}$, s.t. the object which is assigned rank $j$ by $\sigma_a$ is the same object assigned rank $1$ by $\tilde{\sigma}_{j(a)}$, i.e., $\sigma_a(j)= \tilde{\sigma}_{j(a)}(1)$.

3. In Equality 3, $(S^{\top}_{\tilde{\sigma}^{-1}_{j(a)}})_{R_b,2}$ indicates the $R_b$th row and 2nd column of (transposed) signal matrix $S_{\tilde{\sigma}_{j(a)}}$, corresponding to learner action $\tilde{\sigma}_{j(a)}$. Equality 3 holds because $R_b(\tilde{\sigma}_{j(a)}(1))$ is the entry in the row corresponding to action $\tilde{\sigma}_{j(a)}$ and column corresponding to action $R_b$ of $H$ (see Definition 2). 

4. Equality 4 holds from the observation that for a particular $j$, $[(S^{\top}_{\tilde{\sigma}_{j(a)}})_{R_1,2}, (S^{\top}_{\tilde{\sigma}_{j(a)}})_{R_2,2}, \ldots,\\ (S^{\top}_{\tilde{\sigma}_{j(a)}})_{R_{2^m},2}]$ forms the 2nd column of  $(S^{\top}_{\tilde{\sigma}_{j(a)}})$, i.e.,  $(S^{\top}_{\tilde{\sigma}_{j(a)}})_{:,2}$.
\end{proof}

\subsubsection{Local Observability}
\label{local-NC}

{\bf Definition 4}: Two Pareto-optimal (learner's) actions $i$ and $j$ are called \emph{neighboring actions} if $C_i \cap C_j$ is a $(2^m-2)$ dimensional polytope (where $C_i$ is probability cell of action $\sigma_i$). The \emph{neighborhood action set} of two neighboring (learner's) actions $i$ and $j$ is defined as $N^{+}_{i,j}= \{ k \in [m!]: C_i \cap C_j \subseteq C_k\}$. 

{\bf Definition 5}: A pair of neighboring (learner's) actions $i$ and $j$ is said to be locally observable if  $\ell_i -\ell_j \in \oplus_{k \in N^{+}_{i,j}} Col(S_k^{\top})$. The condition of \emph{local observability} holds if every pair of neighboring (learner's) actions is locally observable. 

We now show that local observability condition fails for $L, H$ under SumLoss. First, we present the following two lemmas characterizing Pareto-optimal actions and neighboring actions for SumLoss.

\begin{lem}
\label{pareto-optimal}
For SumLoss, each of learner's action $\sigma_i$ is Pareto-optimal, where Pareto-optimality has been defined in Definition 1.
\end{lem}

\begin{proof}
For any $p \in \Delta$, we have $\ell_i \cdot p= \sum_{j=1}^{2^m}\  p_j \ (\sigma^{-1}_i \cdot R_j) = \sigma^{-1}_i \cdot (\sum_{j=1}^{2^m} p_j R_j)= \sigma^{-1}_i \cdot \E[R]$, where the expectation is taken w.r.t. $p$. By dot product rule between 2 vectors, $l_i \cdot p$ is minimized when ranking of objects according to $\sigma_i$ and expected relevance of objects are in opposite order. That is, the object with highest expected relevance is ranked $1$ and so on. Formally, $l_i \cdot p$ is minimized when $\E[R(\sigma_i(1))] \ge \E[R(\sigma_i(2))] \ge \ldots \ge \E[R(\sigma_i(m))]$. 

Thus, for action $\sigma_i$, probability cell is defined as $C_i= \{p \in \Delta: \sum_{j=1}^{2^m} p_j =1, \ \E[R(\sigma_i(1))] \ge \E[R(\sigma_i(2))] \ge \ldots \ge \E[R(\sigma_i(m))]\}$. Note that, $p \in C_i$ iff action $i$ is optimal w.r.t. $p$.  Since $C_i$ is obviously non-empty and it has only 1 equality constraint (hence $2^m -1$ dimensional), action $i$ is Pareto optimal. 

The above holds true for all learner's actions $\sigma_i$. 
\end{proof}

\begin{lem}
\label{neighbor-actions}
A pair of learner's actions $\{\sigma_i, \ \sigma_j\}$ is a neighboring actions pair, if there is exactly one pair of objects, numbered \{a, b\}, whose positions differ in $\sigma_i$ and $\sigma_j$. Moreover, $a$ needs to be placed just before $b$ in $\sigma_i$ and $b$ needs to placed just before $a$ in $\sigma_j$. 
\end{lem}

\begin{proof}
From Lemma \ref{pareto-optimal}, we know that every one of learner's actions is Pareto-optimal and $C_i$, associated with action $\sigma_i$, has structure $C_i= \{p \in \Delta: \sum_{j=1}^{2^m} p_j =1, \ \E[R(\sigma_i(1))] \ge \E[R(\sigma_i(2))] \ge \ldots \ge \E[R(\sigma_i(m))]\}$.

Let $\sigma_i(k)= a, \ \sigma_i(k+1)=b$. Let it also be true that $\sigma_j(k)= b, \ \sigma_j(k+1)=a$ and $\sigma_i(n)= \sigma_j(n), \ \forall n \neq \{k, \ k+1\}$. Thus, objects in $\{\sigma_i,\sigma_j\}$ are same in all places except in a pair of consecutive places where the objects are interchanged.

Then, $C_i \cap C_j= \{p \in \Delta: \sum_{j=1}^{2^m} p_j =1, \ \E[R(\sigma_i(1)] \ge \ldots \ge \E[R(\sigma_i(k)] =  \E[R(\sigma_i(k+1)] \ge \ldots \ge \E[R(\sigma_i(m)]\}$. Hence, there are two equalities in the non-empty set $C_i \cap C_j$ and it is an $(2^m -2)$ dimensional polytope. Hence condition of Definition 4 holds true and $\{\sigma_i,\sigma_j\}$ are neighboring actions pair.

\end{proof}

Lemma \ref{pareto-optimal} and \ref{neighbor-actions} lead to following result.
\begin{thm}
\label{localobservability}
The local observability condition, as per Definition 5, fails w.r.t. loss matrix $L$ and feedback matrix $H$ defined for SumLoss, already at $m=3$.
\end{thm}

\subsection{Minimax Regret Bound}
\label{minimax-NC}

We establish the minimax regret for SumLoss by combining results on global and local observability. First, we get a lower bound by combining our Theorem~\ref{localobservability} with Theorem 4 of \cite{bartok2013}.
\begin{cor}
\label{sumlossminimaxregret}
Consider the online game for SumLoss with top-$1$ feedback and $m=3$. Then, for every learner's algorithm, there is an adversary strategy generating relevance vectors, such that the expected regret of the learner is $\Omega(T^{2/3})$.
\end{cor}

The fact that the game is globally observable ( Theorem~\ref{globalobservability}), combined with Theorem 3.1 in \cite{cesa2006}, gives an algorithm (inspired by the algorithm originally given in \cite{piccolboni2001}) obtaining $O(T^{2/3})$ regret. 

\begin{cor}
The algorithm in Figure 1 of \cite{cesa2006} achieves $O(T^{2/3})$ regret bound for SumLoss.
\end{cor}

However, the algorithm in \cite{cesa2006} is intractable in our setting since the algorithm necessarily enumerates all the actions of the learner in each round, which is exponential in $m$ in our case ($m!$ to be exact). Moreover, the regret bound of the algorithm also has a linear dependence on the number of actions, which renders the bound useless.\\

{\bf Discussion}: The results above establish that the minimax regret for SumLoss, \emph{restricted to top-$1$ feedback}, is $\Theta(T^{2/3})$. Theorem 4 of \cite{bartok2013} says the following: A partial monitoring game which is both globally and locally observable has minimax regret $\Theta(T^{1/2})$, while a game which is globally observable but \emph{not} locally observable has minimax regret $\Theta(T^{2/3})$. In Theorem ~\ref{globalobservability}, we proved global observability, when feedback is restricted to relevance of top ranked item. The global observability result automatically extends to  feedback on top $k$ items, for $k>1$. This is because for top $k$ feedback, with $k>1$, the learner receives strictly greater information at end of each round than top $1$ feedback (for example, the learner can just throw away relevance feedback on items ranked $2$nd onwards). So, with top $k$ feedback, for general $k$, the game will remain at least globally observable. In fact, our algorithm in the next section will achieve $O(T^{2/3})$ regret bound for SumLoss with top $k$ feedback, $k \ge 1$. However, the same is not true for failure of local observability. Feedback on more than top ranked item can make the game strictly easier for the learner and may make local observability condition hold, for some $k>1$. In fact, for $k=m$ (full feedback), the game will be a simple bandit game (disregarding computational complexity), and hence locally observable. 

\subsection{Algorithm for Obtaining Minimax Regret under SumLoss with Top $k$ Feedback}
\label{algorithm-NC}
We first provide a general algorithmic framework for getting an $O(T^{2/3})$ regret bound for SumLoss, with feedback on top $k$ ranked items per round, for $k \ge 1$. We then instantiate a specific algorithm, which spends $O(m \log m)$ time per round  (thus, highly efficient) and obtains a regret of  rate $O (poly(m)\ T^{2/3})$.
\subsubsection{General Algorithmic Framework}
Our algorithm combines \emph{blocking} with a randomized \emph{full information} algorithm. We first divide time horizon $T$ into blocks (referred to as blocking). Within each block, we allot a small number of rounds for pure \emph{exploration}, which allows us to estimate the average of the full relevance vectors generated by the adversary in that block. The estimated average vector is cumulated over blocks and then fed to a full information algorithm for the next block. The randomized full information algorithm \emph{exploits} the information received at the beginning of the block to maintain distribution over permutations (learner's actions). In each round in the new block, actions are chosen according to the distribution and presented to the user. 

The \emph{key property} of the randomized full information algorithm is this: for any online game in an adversarial setting played over $T$ rounds, if the loss of each action is known at end of each round (full information), the algorithm should have an expected regret rate of $O(C \sqrt{T})$, where the regret is the difference between cumulative loss of the algorithm and cumulative loss of best action in hindsight, and $C$ is a parameter specific to the full information algorithm. 
 
Our algorithm is motivated by the reduction from bandit-feedback to full feedback scheme given in \cite{nisan2007}. However, the reduction \emph{cannot be directly applied to our problem}, because we are not in the bandit setting and hence do not know loss of any action. Further, the algorithm of \cite{nisan2007} necessarily spends $N$ rounds per block to try out \emph{each} of the $N$ available actions --- this is impractical in our setting since $N = m!$. 

Algorithm~\ref{alg:top-k} describes our approach. A key aspect is the formation of estimate of average relevance vector of a block (line 16), for which we have the following lemma:

\begin{lem}
\label{unbiasedestimator}
Let the average of (full) relevance vectors over the time period $\{1,2,\ldots,t\}$ be denoted as $R_{1:t}^{avg}$, that is, $R_{1:t}^{avg}= \sum_{n=1}^{t}\dfrac{R_n}{t} \in \mathbb{R}^m$. Let $\{i_1,i_2,\ldots,i_{\ceil{m/k}}\}$ be $\ceil{m/k}$ arbitrary time points, chosen uniformly at random, without replacement, from $\{1,\ldots,t\}$. At time point $i_j$, only $k$ distinct components of relevance vector $R_{i_j}$, i.e., $\{R_{i_j}(k\cdot(j-1)+1), R_{i_j}(k\cdot(j-1)+2), \ldots, R_{i_j}(k\cdot j)\}$, becomes known, $\forall j \in \{1,\ldots,\ceil{m/k}\}$ (for $j= \ceil{m/k}$, there might be less than $k$ components available). Then the vector formed from the $m$ revealed components, i.e. $\hat{R}_{t}= [R_{i_j}(k\cdot(j-1)+1), R_{i_j}(k\cdot(j-1)+2), \ldots, R_{i_j}(k\cdot j)]_{\{j = 1,2, \ldots, \ceil{m/k}\}}$ is an unbiased estimator of $R_{1:t}^{avg}$.
\end{lem}

\begin{proof}
We can write $\hat{R}_{t}= \sum_{j=1}^{\ceil{m/k} } \sum_{\ell=1}^k R_{i_j}(k\cdot(j-1) + \ell) e_{k\cdot(j-1) + \ell}$, where $e_i$ is the $m$ dimensional standard basis vector along coordinate $j$. Then, taking expectation over the randomly chosen time points, we have: $E_{i_1,\ldots,i_{\ceil{m/k}}}(\hat{R}_{t})= \sum_{j=1}^{\ceil{m/k} } E_{i_j} [ \sum_{\ell=1}^k R_{i_j}(k\cdot(j-1) + \ell) e_{k \cdot(j-1) + \ell}]= \sum_{j=1}^{m} \sum_{\ell=1}^k \sum_{n=1}^{t}\dfrac{R_n(k\cdot(j-1) + \ell) e_{k\cdot(j-1) + \ell}}{t}$= $R_{1:t}^{avg}$.
\end{proof}

Suppose we have a full information algorithm whose regret in $t$ rounds is upper bounded by $C\sqrt{T}$ for some constant $C$ and let $C^I$ be the maximum loss that the learner can suffer in a round. Note that $C^I$ depends on the loss used and on the range of the relevance scores. We have the following regret bound, obtained from application of Algorithm~\ref{alg:top-k} on SumLoss with top $k$ feedback.

\begin{thm}
\label{regret}
Let $C, C^I$ be the constants defined above. The expected regret under SumLoss, obtained by applying Algorithm \ref{alg:top-k}, with relevance feedback on top $k$ ranked items per round ($k \ge 1$), and the expectation being taken over randomized learner's actions $\sigma_t$, is
\begin{equation}
\E\left[\sum_{t=1}^T SumLoss(\sigma_t, R_t)\right] - \underset{\sigma}{\min}\sum_{t=1}^T SumLoss(\sigma_t, R_t) \le C^{I} \ceil{m/k} K + C \frac{T}{\sqrt{K}} .
\end{equation}
Optimizing over block size $K$, the final regret bound is:
\begin{equation}
\E\left[\sum_{t=1}^T SumLoss(\sigma_t, R_t)\right] - \underset{\sigma}{\min}\sum_{t=1}^T SumLoss(\sigma_t, R_t) \le 2 (C^I)^{1/3} C^{2/3} \ceil{m/k}^{1/3} T^{2/3} .
\end{equation}

\end{thm}



\floatstyle{ruled}
\newfloat{algorithm}{htbp}{loa}
\floatname{algorithm}{Algorithm}
\begin{algorithm}
{\small
\caption{RankingwithTop-kFeedback(RTop-kF)- Non Contextual}
\label{alg:top-k}
\begin{tabbing}
tabs \= tabs \= tabs \= tabs \kill
1: $T=$ Time horizon, $K=$ No. of (equal sized) blocks, FI= randomized full information algorithm.\\
2: Time horizon divided into equal sized blocks $\{B_1,\ldots,B_K\}$, where $B_i= \{(i-1)(T/K)+1, \ldots, i (T/K)\}$. \\
3: Initialize $\hat{s}_0=\mathbf{0} \in \mathbb{R}^m$. Initialize any other parameter specific to FI. \\
4: {\bf For} \= $i= 1,\ldots,K$\\
5:\> Select $\ceil{m/k}$ time points $\{i_1,\ldots,i_{\ceil{m/k}}\}$ from block $B_i$, uniformly at random, without replacement.\\
6:\> Divide the $m$ items into $\ceil{m/k}$ cells, with $k$ distinct items in each cell. \footnotemark \\
7:\> {\bf For} \=  $t \in B_i$\\
8:\>\> {\bf If} $t = i_j \in \{i_1,\ldots,i_{\ceil{m/k}}\}$\\
9:\>\>\> {\bf Exploration round}:\\
10:\>\>\> Output any permutation $\sigma_t$ which places items of $j$th cell in top $k$ positions (in any order).\\
11:\>\>\> Receive feedback as relevance of top $k$ items of $\sigma_t$ (i.e., items of $j$th cell).\\
12:\>\> {\bf Else} \\
13:\>\>\> {\bf Exploitation round}:\\
14:\>\>\> Feed $\hat{s}_{i-1}$ to the randomized full information algorithm FI and output $\sigma_t$ according to FI.\\
15:\> {\bf end for}\\
16:\> Set $\hat{R}_i \in \mathbb{R}^m$  as vector of relevances of the $m$ items collected during exploration rounds.\\
17:\> Update $\hat{s}_i= \hat{s}_{i-1} + \hat{R}_{i}$.\\
18: {\bf end for}
\end{tabbing}
}
\end{algorithm}

\footnotetext{For e.g., assume $m=7$ and $k=2$. Then place items $(1,2)$ in cell $1$, items $(3,4)$ in cell $2$, items $(5,6)$ in cell $3$ and item $7$ in cell $4$.}

\subsubsection{Computationally Efficient Algorithm with FTPL}

We instantiate our general algorithm with \emph{Follow The Perturbed Leader} (FTPL) full information algorithm \citep{kalai2005}. 
The following modifications are needed in Algorithm ~\ref{alg:top-k} to implement FTPL as the full information algorithm:

{\bf Initialization of parameters}: In line 3 of the algorithm, the parameter specific to FTPL is randomization parameter $\epsilon \in \mathbb{R}$.

{\bf Exploitation round}: $\sigma_t$, during exploitation, is sampled by FTPL as follows:  sample $p_t \in [0,1/\epsilon]^m$ from the product of uniform distribution in each dimension. Output permutation $\sigma_{t}= M(\hat{s}_{i-1} + p_t)$  where $M(y)= \underset{\sigma}{\argmin} \ \sigma^{-1}\cdot y$. \\

{\bf Discussion}: The key reason for using FTPL as the full information algorithm is that the structure of our problem allows the permutation $\sigma_t$ to be chosen during exploitation round via a simple sorting operation on $m$ objects. This leads to an easily implementable algorithm which spends only $O(m \log m)$ time per round (sorting is in fact the most expensive step in the algorithm). The reason that the simple sorting operation does the trick is the following: FTPL only \emph{implicitly} maintains a distribution over $m!$ actions (permutations) at beginning of each round. Instead of having an explicit probability distribution over each action and sampling from it, FTPL mimics sampling from a distribution over actions by randomly perturbing the information vector received so far (say $\hat{s}_{i-1}$ in block $B_i$)  and then sorting the items by perturbed score. The random perturbation puts an implicit weight on each of the $m!$ actions and sorting is basically sampling according to the weights. This is an advantage over general full information algorithms based on exponential weights, which maintain explicit weight on actions and samples from it.

We have the following corollary:
\begin{cor}
\label{efficientregret-SumLoss}
The expected regret of SumLoss, obtained by applying Algorithm \ref{alg:top-k}, with FTPL full information algorithm and feedback on top $k$ ranked items at end of each round ($k \ge 1$), and $K= O\left(\dfrac{m^{1/3} T^{2/3}}{\ceil{m/k}^{2/3}}\right)$, $\epsilon= O(\frac{1}{\sqrt{m K}})$, is: 
\begin{equation}
\E\left[\sum_{t=1}^T SumLoss(\sigma_t, R_t)\right] -  \underset{\sigma}{\min}\sum_{t=1}^T SumLoss(\sigma_t, R_t) \le  O(m^{7/3} \ceil{m/k}^{1/3} T^{2/3}) .
\end{equation}
where $O(\cdot)$ hides some numeric constants.
\end{cor}
Assuming that $\ceil{m/k} \sim m/k$, the regret rate in Corollary~\ref{efficientregret-SumLoss} is $O \left( \dfrac{m^{8/3}T^{2/3}}{k^{1/3}} \right)$
\subsection{Regret Bounds for PairwiseLoss, DCG and Precision@n}
{\bf PairwiseLoss}: As we saw in Eq.~\ref{rankloss-sumloss}, the regret of SumLoss is same as regret of PairwiseLoss. Thus, SumLoss in Corollary~\ref{efficientregret-SumLoss} can be replaced by PairwiseLoss to get exactly same result.\\

{\bf DCG}: All the results of SumLoss can be extended to DCG (see Appendix A). Moreover, the results can be extended even for multi-graded relevance vectors. Thus, the minimax regret under DCG, \emph{restricted to feedback on top ranked item}, even when the adversary can play multi-graded relevance vectors, is $\Theta(T^{2/3})$. 

The main differences between SumLoss and DCG are the following. The former is a loss function; the latter is a gain function. Also, for DCG, $f(\sigma) \neq \sigma^{-1}$  (see definition in Sec.\ref{rankingmeasures-NC} ) and when relevance is multi-graded, DCG cannot be expressed as $f(\sigma) \cdot R$,  as clear from definition. Nevertheless, DCG can be expressed as $f(\sigma) \cdot g(R)$, , where $g(R)= [g^s(R(1)), g^s(R(2)), \ldots, g^s(R(m))], \ g^s(i)= 2^i -1$ is constructed from univariate, monotonic, scalar valued functions ($g(R)= R$ for binary graded relevance vectors). Thus, Algorithm~\ref{alg:top-k} can be applied (with slight variation), with FTPL full information algorithm and top $k$ feedback, to achieve regret of $O(T^{2/3})$. The slight variation is that during \emph{exploration} rounds, when relevance feedback is collected to form the estimator at end of the block,  the relevances should be transformed by function $g^s(\cdot)$. The estimate is then constructed in the transformed space and fed to the full information algorithm. In the \emph{exploitation} round, the selection of $\sigma_t$ remains exactly same as in SumLoss, i.e.,  $\sigma_{t}= M(\hat{s}_{i-1} + p_t)$  where $M(y)= \underset{\sigma}{\argmin} \ \sigma^{-1}\cdot y$. This is because  $\underset{\sigma}{\argmax}\ f(\sigma) \cdot y= \underset{\sigma}{\argmin}\ \sigma^{-1} \cdot y$, by definition of $f(\sigma)$ in DCG.

Let relevance vectors chosen by adversary have $n+1$ grades, i.e., $R \in \{0,1,\ldots, n\}^m$. In practice, $n$ is almost always less than $5$. We have the following corollary:
\begin{cor}
\label{efficientregret-DCG}
The expected regret of DCG, obtained by applying Algorithm \ref{alg:top-k}, with FTPL full information algorithm and feedback on top $k$ ranked items at end of each round ($k \ge 1$), and $K= O\left(\dfrac{m^{1/3} T^{2/3}}{\ceil{m/k}^{2/3}}\right)$, $\epsilon= O(\frac{1}{(2^n -1)^2 \sqrt{m K}})$, is: 
\begin{equation}
 \underset{\sigma}{\max}\sum_{t=1}^T DCG(\sigma_t, R_t) - \E\left[\sum_{t=1}^T DCG(\sigma_t, R_t)\right]  \le  O((2^n-1) m^{4/3} \ceil{m/k}^{1/3} T^{2/3}) .
\end{equation}

\end{cor}
Assuming that $\ceil{m/k} \sim m/k$, the regret rate in Corollary~\ref{efficientregret-DCG} is $O \left( \dfrac{(2^n-1)m^{5/3}T^{2/3}}{k^{1/3}} \right)$.
{\bf Precision@n}: Since Precision@$n = f(\sigma) \cdot R$, the global observability property of SumLoss can be easily extended to it and Algorithm~\ref{alg:top-k} can be applied, with FTPL full information algorithm and top $k$ feedback, to achieve regret of $O(T^{2/3})$. In the \emph{exploitation} round, the selection of $\sigma_t$ remains exactly same as in SumLoss, i.e.,  $\sigma_{t}= M(\hat{s}_{i-1} + p_t)$  where $M(y)= \underset{\sigma}{\argmin} \ \sigma^{-1}\cdot y$. 

However, the local observability property of SumLoss does not extend to Precision@$n$. The reason is that while $f(\cdot)$ of SumLoss is strictly monotonic, $f(\cdot)$ of Precision@$n$ is monotonic but not strict. Precision@$n$ depends only on the objects in the top $n$ positions of the ranked list, \emph{irrespective of the order}. A careful review shows that Lemma~\ref{neighbor-actions} fails to extend to the case of Precision@$n$, due to lack of strict monotonicity. Thus, we cannot define the neighboring action set of the Pareto optimal action pairs, and hence cannot prove or disprove local observability. 
 
We have the following corollary:
\begin{cor}
\label{efficientregret-Precision@n}
The expected regret of Precision@$n$, obtained by applying Algorithm \ref{alg:top-k}, with FTPL full information algorithm and feedback on top $k$ ranked items at end of each round ($k \ge 1$), and $K= O\left(\dfrac{m^{1/3} T^{2/3}}{\ceil{m/k}^{2/3}}\right)$, $\epsilon= O(\frac{1}{\sqrt{m K}})$, is: 
\begin{equation}
 \underset{\sigma}{\max}\sum_{t=1}^T Precision@n(\sigma_t, R_t) - \E\left[\sum_{t=1}^T Precision@n(\sigma_t, R_t)\right]  \le  O(n\ m^{1/3} \ceil{m/k}^{1/3} T^{2/3}) .
\end{equation}

\end{cor}
Assuming that $\ceil{m/k} \sim m/k$, the regret rate in Corollary~\ref{efficientregret-Precision@n} is $O \left(\dfrac{n\ m^{2/3}T^{2/3}}{k^{1/3}} \right)$.
\subsection{Non-Existence of Sublinear Regret Bounds for NDCG, AP and AUC}
As stated in Sec.~\ref{rankingmeasures-NC}, NDCG, AP and AUC are normalized versions of measures DCG, Precision@$n$ and PairwiseLoss. We have the following lemma for all these normalized ranking measures.

\begin{lem}\label{globalfailsfornormalized}
The global observability condition, as per Definition 1, fails for NDCG, AP and AUC, when feedback is restricted to top ranked item.
\end{lem}

Combining the above lemma with Theorem 2 of \cite{bartok2013}, we conclude that there \emph{cannot exist any algorithm which has sub-linear regret for any of the following measures: NDCG, AP or AUC, when restricted to top $1$ feedback}.
\begin{thm}
There exists an online game, for NDCG with top-$1$ feedback, such that for every learner's algorithm, there is an adversary strategy generating relevance vectors, such that the expected regret of the learner is $\Omega(T)$. Furthermore, the same lower bound holds if NDCG is replaced by AP or AUC.
\end{thm}

%% file: OnlineRankingC.tex
\section{Online Ranking with Restricted Feedback- Contextual Setting}
\label{C}
All proofs not in the main text are in Appendix B.
\subsection{Problem Setting and Learning to Rank Algorithm}
First, we introduce some additional notations to Section~\ref{preliminaries-NC}. In the contextual setting, each query and associated items (documents) are represented jointly as a feature matrix. Each feature matrix, $X \in \mathbb{R}^{m \times d}$, consists of a list of $m$ documents, each represented as a feature vector in $\mathbb{R}^d$. The feature matrices are considered side-information (context) and represents varying items, as opposed to the fixed set of items in the first part of our work. $X_{i:}$ denotes $i$th row of $X$.  We assume feature vectors representing documents are bounded by $R_D$ in $\ell_2$ norm. The relevance vectors are same as before.

As per traditional learning to rank setting with query-document matrices, documents are ranked by a ranking function. The prevalent technique is to represent a ranking function as a scoring function and get ranking by sorting scores in descending order. A linear scoring function produces score vector as $f_{w}(X)= Xw= s^w \in \mathbb{R}^m$, with $w \in \mathbb{R}^d$. Here, $s^w(i)$ represents score of $i$th document ($s^w$ points to score $s$ being generated by using parameter $w$). We assume that ranking parameter space is bounded in $\ell_2$ norm, i.e, $\|w\|_2 \le U$, $\forall \ w$. $\pi_s = \argsort(s)$ is the permutation induced by sorting score vector $s$ in descending order. As a reminder, a permutation $\pi$ gives a mapping from ranks to documents and $\pi^{-1}$ gives a mapping from documents to ranks. 

Performance of ranking functions are judged, based on the rankings obtained from score vectors, by ranking measures like DCG, AP and others. However, the measures themselves are discontinuous in the score vector produced by the ranking function, leading to intractable optimization problems. Thus, most learning to rank methods are based on minimizing \emph{surrogate} losses, which can be optimized efficiently. A surrogate $\phi$ takes in a score vector $s$ and relevance vector $R$ and produces a real number, i.e., $\phi: \mathbb{R}^m \times \{0,1,\ldots,n \}^m \mapsto \mathbb{R}$. $\phi(\cdot,\cdot)$ is said to be convex if it is convex in its first argument, for any value of the second argument. Ranking surrogates are designed in such a way that the ranking function learnt by optimizing the surrogates has good performance with respect to ranking measures.  

{\bf Formal problem setting}: We formalize the problem as a game being played between a learner and an \emph{oblivious} adversary over $T$ rounds (i.e., an adversary who generates moves without knowledge of the learner's algorithm). The learner's action set is the uncountably infinite set of score vectors in $\mathbb{R}^m$ and the adversary's action set is all possible relevance vectors, i.e., $(n+1)^m$ possible vectors. At round $t$, the adversary generates a list of documents, represented by a matrix $X_t \in \mathbb{R}^{m \times d}$, pertaining to a query (the document list is considered as side information). The learner receives $X_t$, produces a score vector $\tilde{s}_t \in \reals^m$ and ranks the documents by sorting according to score vector. The adversary then generates a relevance vector $R_t$ but only reveals the relevances of top $k$ ranked documents to the learner. The learner uses the feedback to choose its action for the next round (updates an internal scoring function). The learner suffers a loss as measured in terms of a surrogate $\phi$, i.e, $\phi(\tilde{s}_t,R_t)$. As is standard in online learning setting, the learner's performance is measured in terms of its expected regret: 
\begin{equation*}
\E \left[\sum_{t=1}^T \phi(\tilde{s}_t,R_t) \right] - \min_{\|w\|_2 \le U} \sum_{t=1}^T \phi(X_tw, R_t),
\end{equation*} 
where the expectation is taken w.r.t. to randomization of learner's strategy and $X_tw = s_t^w$ is the score produced by the linear function parameterized by $w$.

\floatstyle{ruled}
\newfloat{algorithm}{htbp}{loa}
\floatname{algorithm}{Algorithm}
\begin{algorithm*}
\caption{Ranking with Top-k Feedback (RTop-kF)- Contextual}
\label{alg:RTop-kF}
\begin{tabbing}
1: Exploration parameter $\gamma \in (0,\frac{1}{2})$, learning parameter $\eta>0$, ranking parameter  $w_1=\mathbf{0} \in \mathbb{R}^d$\\
2: {\bf For} \=$t=1$ to $T$ \\
3: \> Receive $X_t$ (document list pertaining to query $q_t$)\\
4: \> Construct score vector $s_t^{w_t}= X_tw_t$ and get permutation $\sigma_t= \argsort(s_t^{w_t})$ \\
5: \> $\mathbb{Q}_t(s) = (1-\gamma)\delta(s-s^{w_t}_t) + \gamma\text{Uniform}([0,1]^m)$ ($\delta$ is the Dirac Delta function). \\
6: \> Sample $\tilde{s}_t \sim \mathbb{Q}_t$ and output the ranked list $\tilde{\sigma}_t = \argsort(\tilde{s}_t)$ \\
   \> (Effectively, it means $\tilde{\sigma}_t$ is drawn from $\mathbb{P}_t(\sigma)= (1-\gamma)\mathbbm{1}(\sigma=\sigma_t) + \frac{\gamma}{m!}$) \\
7: \> Receive relevance feedback on top-$k$ items, i.e., ($R_t( \tilde{\sigma}_t(1)), \ldots, R_t(\tilde{\sigma}_t(k))$)\\
8: \> Suffer loss $\phi(\tilde{s}_t,R_t)$ (Neither loss nor $R_t$ revealed to learner) \\
9: \> Construct $\tilde{z}_t$, an unbiased estimator of gradient $\nabla_{w=w_t} \phi(X_t w,R_t)$, from  top-$k$ feedback.\\ 
10: \> Update $w= w_t - \eta \tilde{z}_t$ \\
11: \> $w_{t+1}= \min \{1, \frac{\text{U}}{\|w\|_2}\} w$ (Projection onto Euclidean ball of radius $U$).\\
12: {\bf End For}
\end{tabbing}
\end{algorithm*} 

{\bf Relation between feedback and structure of surrogates:}
Algorithm~\ref{alg:RTop-kF} is our general algorithm for learning a ranking function, online, from partial feedback. 
The key step in Algorithm~\ref{alg:RTop-kF} is the construction of the unbiased estimator $\tilde{z}_t$ of the surrogate gradient $\nabla_{w=w_t} \phi(X_t w, R_t)$. The information present for the construction process, at end of round $t$, is the random score vector $\tilde{s}_t$ (and associated permutation $\tilde{\sigma}_t$) and relevance of top-$k$ items of $\tilde{\sigma}_t$, i.e., $\{R_t( \tilde{\sigma}_t(1)), \ldots, R_t(\tilde{\sigma}_t(k)\}$. Let $\E_{t}\left[\cdot\right]$ be the expectation operator w.r.t. to randomization at round $t$, conditioned on $(w_1,\ldots, w_t)$. Then $\tilde{z}_t$ being an unbiased estimator of gradient of surrogate, w.r.t $w_t$,  means the following: $\E_{t} \left[\tilde{z}_t \right]= \nabla_{w=w_t} \phi(X_tw, R_t)$. We note that conditioned on the past, the score vector $s^{w_t}_t = X_t w_t$ is deterministic. We start with a general result relating feedback to the construction of unbiased estimator of a vector valued function. Let $\mathbb{P}$ denote a probability distribution on $S_m$, i.e, $\sum_{\sigma \in S_m} \mathbb{P}(\sigma)= 1$. For a distinct set of indices $(j_1,j_2,\ldots, j_k)$ $\subseteq$ $[m]$, we denote $p(j_i,j_2,\ldots,j_k)$ as the the sum of probability of permutations whose first $k$ objects match objects $(j_1, \ldots, j_k)$, in order. Formally,
\begin{equation}
\label{eq:shortprob}
\begin{split}
p(j_1,  \ldots, j_k) = 
 \sum\limits_{\pi \in S_m}\mathbb{P}(\pi) \mathbbm{1}(\pi(1)= j_1,\ldots,\pi(k)=j_k) .
\end{split}
\end{equation}
We have the following lemma relating feedback and structure of surrogates:
\begin{lem}
\label{unbiasedestimator}
Let $F: \mathbb{R}^m \mapsto \mathbb{R}^a$ be a vector valued function, where $m\ge 1$, $a\ge 1$. For a fixed $x \in \mathbb{R}^m$, let $k$ entries of $x$ be observed at random. That is, for a fixed probability distribution $\mathbb{P}$ and some random $\sigma \sim \mathbb{P}(S_m)$, observed tuple is $\{\sigma, x_{\sigma(1)}, \ldots, x_{\sigma(k)}\}$. A necessary condition for existence of an unbiased estimator of $F(x)$, that can be constructed from $\{\sigma, x_{\sigma(1)}, \ldots, x_{\sigma(k)}\}$, is that it should be  possible to decompose $F(x)$ over $k$ (or less) coordinates of $x$ at a time. That is, $F(x)$ should have the structure:
\begin{equation}
\label{eq:decoupling}
F(x)= \sum\limits_{(i_1,i_2,\ldots,i_{\ell}) \in \ \perm{m}{\ell}} h_{i_1,i_2,\ldots,i_{\ell}}(x_{i_1}, x_{i_2},\ldots, x_{i_{\ell}}) 
\end{equation}
where $\ell \le k$, $\perm{m}{\ell}$ is $\ell$ permutations of $m$ and $h: \mathbb{R}^{\ell} \mapsto \mathbb{R}^a$ (the subscripts in $h$ are used to denote possibly different functions in the decomposition structure).
Moreover, when $F(x)$ can be written in form of Eq~\ref{eq:decoupling} , with $\ell=k$, an unbiased estimator of $F(x)$, based on $\{\sigma, x_{\sigma(1)}, \ldots, x_{\sigma(k)}\}$, is, 
\begin{equation}
\label{eq:unbiasedestimator}
\begin{split}
g(\sigma, & x_{\sigma(1)}, \ldots, x_{\sigma(k)})=\\
& \dfrac{\sum\limits_{(j_1,j_2,\ldots, j_k) \in S_k}h_{\sigma(j_1),\ldots, \sigma(j_k)}(x_{\sigma(j_1)}, \ldots, x_{\sigma(j_k)})}{\sum\limits_{\substack{(j_1,\ldots, j_k) \in S_k}} p(\sigma(j_1),\ldots,\sigma(j_k))} 
\end{split}
\end{equation}
where $S_k$ is the set of $k!$ permutations of $[$k$]$ and $p(\sigma(1),\ldots, \sigma(k))$ is as in Eq~\ref{eq:shortprob} .
\end{lem}

{\bf Illustrative Examples:} We provide simple examples to concretely illustrate the abstract functions in Lemma~\ref{unbiasedestimator}. Let $F(\cdot)$ be the identity function, and $x \in \reals^m$.  Thus, $F(x)=x$ and the function decomposes over $k=1$ coordinate of x as follows: $F(x)= \sum_{i=1}^m x_i e_i$, where $e_i \in \mathbb{R}^m$ is the standard basis vector along coordinate $i$. Hence, $h_{i}(x_{i})= x_{i} e_{i}$. Based on top-1 feedback, following is an unbiased estimator of $F(x)$: $g(\sigma, x_{\sigma(1)})= \dfrac{x_{\sigma(1)}e_{\sigma(1)}}{p(\sigma(1))}$,
where $p(\sigma(1))=  \sum\limits_{\pi \in S_m}\mathbb{P}(\pi) \mathbbm{1}(\pi(1)= \sigma(1))$. In another example, let $F: \reals^3 \mapsto \reals^2$ and $x \in \reals^3$. Let $F(x)= [x_1+ x_2; x_2+x_3]^{\top}$. Then the function decomposes over $k=1$ coordinate of $x$ as $F(x)= x_1 e_1 + x_2 (e_1 + e_2) + x_3 e_2$, where $e_i \in \reals^2$. Hence, $h_1 (x_1) = x_1 e_1$, $h_2(x_2)= x_2(e_1 + e_2)$ and $h_3(x_3)= x_3 e_2$. An unbiased estimator based on top-1 feedback is: $g(\sigma, x_{\sigma(1)})= \dfrac{h_{\sigma(1)}(x_{\sigma(1)})}{p(\sigma(1))}$.

\subsection{Unbiased Estimators of Gradients of Surrogates}
Algorithm~\ref{alg:RTop-kF} can be implemented for any ranking surrogate as long as an unbiased estimator of the gradient can be constructed from the random feedback. We will use techniques from \emph{online convex optimization} to obtain formal regret guarantees. We will thus construct the unbiased estimator of four major ranking surrogates. Three of them are popular \emph{convex} surrogates, one each from the three major learning to rank methods, i.e., \emph{pointwise}, \emph{pairwise} and \emph{listwise} methods. The fourth one is a popular \emph{non-convex} surrogate.

{\bf Shorthand notations:} We note that by chain rule, $\nabla_{w=w_t} \phi(X_t w,R_t)$= $X_t^{\top} \nabla_{s^{w_t}_t} \phi(s^{w_t}_t,R_t)$, where $s_t^{w_t}= X_t w_t$.  Since $X_t$ is deterministic in our setting, we focus on unbiased estimators of $\nabla_{s^{w_t}_t} \phi(s^{w_t}_t,R_t)$ and take a matrix-vector product with $X_t$. To reduce notational clutter in our derivations, we drop $w$ from $s^w$ and the subscript $t$ throughout. Thus, in our derivations, $\tilde{z}= \tilde{z}_t$, $X=X_t$, $s= s_t^{w_t}$ (and not $\tilde{s}_t$), $\sigma= \tilde{\sigma}_t$ (and not $\sigma_t$), $R= R_t$, $e_i$ is standard basis vector in $\mathbb{R}^m$ along coordinate $i$ and $p(\cdot)$ as in Eq.~\ref{eq:shortprob} with $\mathbb{P}= \mathbb{P}_t$ where $\mathbb{P}_t$ is the distribution in round $t$ in Algorithm~\ref{alg:RTop-kF}.

\subsubsection{ Convex Surrogates}
\label{convexsurrogates}
{\bf Pointwise Method:} We will construct the unbiased estimator of the gradient of squared loss \citep{cossock2006subset}: $\phi_{sq}(s,R)= \|s-R\|_2^2$. The gradient $\nabla_s \phi_{sq}(s,R)$ is $2(s-R) \in \mathbb{R}^m$.  As we have already demonstrated in the example following Lemma~\ref{unbiasedestimator}, \emph{we can construct unbiased estimator of $R$ from top-1 feedback }($\{\sigma,R(\sigma(1))\}$). Concretely, the unbiased estimator is:
\[
\tilde{z}=  X^{\top} \left(2 \left(s  - \dfrac{R(\sigma(1)) e_{\sigma(1)}}{p(\sigma(1))} \right) \right) .
\] 

{\bf Pairwise Method:} We will construct the unbiased estimator of the gradient of hinge-like surrogate in RankSVM \citep{joachims2002}: $\phi_{svm}(s,R)= \sum_{i \neq j =1} \mathbbm{1}(R(i)>R(j)) \max (0, 1 + s(j) -s(i))$. The gradient is given by $\nabla_s \phi_{svm}(s,R)=\sum_{i \neq j =1}^m  \mathbbm{1}(R(i)>R(j)) \mathbbm{1}(1+s(j)>s(i)) (e_j- e_i) \in \mathbb{R}^m$. Since $s$ is a known quantity, from Lemma~\ref{unbiasedestimator}, we can construct $F(R)$ as follows: $F(R)= F_s(R)= \sum_{i \neq j =1}^m h_{s, i,j}(R(i),R(j))$, where $ h_{s, i,j}(R(i),R(j))= \mathbbm{1}(R(i) >R(j)) \mathbbm{1}(1+s(j)>s(i)) (e_j - e_i)$.  Since $F_s(R)$ is decomposable over 2 coordinates of $R$ at a time, \emph{we can construct an unbiased estimator from top-2 feedback} ($\{\sigma, R({\sigma(1)}), R(\sigma(2))\}$). The unbiased estimator is:
\[
\begin{split}
\tilde{z}=  X^{\top} \left(\dfrac{h_{s,\sigma(1),\sigma(2)}(R(\sigma(1)), R(\sigma(2))) + h_{s,\sigma(2),\sigma(1)}(R(\sigma(2)), R(\sigma(1)))}{p(\sigma(1),\sigma(2))+ p(\sigma(2),\sigma(1))}\right) .
\end{split}
\]

We note that the unbiased estimator was constructed from top-2 feedback. The following lemma, in conjunction with the necessary condition of Lemma~\ref{unbiasedestimator} shows that it is the minimum information required to construct the unbiased estimator.

\begin{lem}
\label{RankSVM}
The gradient of RankSVM surrogate, i.e., $\phi_{svm}(s,R)$ cannot be decomposed over 1 coordinate of R at a time. 
\end{lem}
 
%
{\bf Listwise Method:}  Convex surrogates developed for listwise methods of learning to rank are defined over the entire score vector and relevance vector. Gradients of such surrogates cannot usually be decomposed over coordinates of the relevance vector. We will focus on the cross-entropy surrogate used in the highly cited ListNet \citep{cao2007learning} ranking algorithm and show how a very natural modification to the surrogate makes its gradient estimable in our partial feedback setting.

The authors of the ListNet method use a cross-entropy surrogate on two probability distributions on permutations, induced by score and relevance vector respectively. More formally, the surrogate is defined as follows\footnote{The ListNet paper actually defines a family of losses based on probability models for top $r$ documents, with $r \le m$. We use $r=1$ in our definition since that is the version implemented in their experimental results.}. Define $m$ maps from $\reals^m$ to $\reals$ as: $P_j(v) = \exp(v(j))/\sum_{j=1}^m \exp(v(j))$ for $j \in [m]$. Then, for score vector $s$ and relevance vector $R$, 
$\listnet(s,R) = - \sum_{i=1}^m P_i(R) \log P_i(s)$ and $\nabla_s \listnet(s,R) = \sum_{i=1}^m \left(- \frac{\exp(R(i))}{\sum_{j=1}^m \exp(R(j))} + \frac{\exp(s(i))}{\sum_{j=1}^m \exp(s(j))} \right) e_i$. We have the following lemma about the gradient of $\phi_{LN}$.
\begin{lem}
\label{listnet}
The gradient of ListNet surrogate $\phi_{LN}(s,R)$ cannot be decomposed over $k$, for $k = 1,2$, coordinates of R at a time.
\end{lem}

In fact, an examination of the proof of the above lemma reveals that decomposability at any $k < m$ does not hold for the gradient of LisNet surrogate, though we only prove it for $k=1,2$ (since feedback for top $k$ items with $k>2$ does not seem practical).  Due to Lemma~\ref{unbiasedestimator}, this means that if we want to run Alg.~\ref{alg:RTop-kF} under top-$k$ feedback, a modification of ListNet is needed. We now make such a modification. 

We first note that the cross-entropy surrogate of ListNet can be easily obtained from a standard divergence, viz. Kullback-Liebler divergence. Let $p, q \in \reals^m$ be 2 probability distributions ($\sum_{i=1}^m p_i= \sum_{i=1}^m q_i= 1$). Then $KL(p,q)= \sum_{i=1}^m p_i \log(p_i) - \sum_{i=1}^m p_i \log(q_i) - \sum_{i=1}^m p_i + \sum_{i=1}^m q_i$. Taking $p_i= P_i(R)$ and $q_i= P_i(s)$, $\forall \ i\in [m]$ (where $P_i(v)$ is as defined in $\listnet$) and noting that $\listnet(s,R)$ needs to be minimized w.r.t. $s$ (thus we can ignore the $\sum_{i=1}^m p_i \log(p_i)$ term in $KL(p,q)$), we get the cross entropy surrogate from KL.

Our natural modification now easily follows by considering KL divergence for \emph{un-normalized} vectors (it should be noted that KL divergence is an instance of a Bregman divergence). Define $m$ maps from $\reals^m$ to $\reals$ as: $P'_j(v) = \exp(v(j))$ for $j \in [m]$. Now define $p_i= P'_i(R)$ and $q_i = P'_i(s)$. Then, the modified surrogate $\phi_{KL}(s,R)$ is:
$$\sum\limits_{i=1}^m e^{R(i)} \log(e^{R(i)}) - \sum\limits_{i=1}^m e^{R(i)} \log(e^{s(i)}) - \sum\limits_{i=1}^m e^{R(i)} + \sum\limits_{i=1}^m e^{s(i)} ,$$
and $\sum\limits_{i=1}^m \left(\exp(s(i)) - \exp(R(i)) \right) e_i$ is its gradient w.r.t. $s$.
Note that $\phi_{KL}(s,R)$ is non-negative and convex in $s$. Equating gradient to ${\bf 0} \in \reals^m$, at the minimum point, $s(i)= R(i), \ \forall \ i \in [m]$. Thus, the sorted order of optimal score vector agrees with sorted order of relevance vector and it is a valid ranking surrogate. 

Now, from Lemma~\ref{unbiasedestimator}, we can construct $F(R)$ as follows: $F(R)= F_s(R)= \sum_{i =1}^m h_{s, i}(R(i))$, where $h_{s, i}(R(i))= \left(\exp(s(i)) - \exp(R(i)) \right) e_i$. Since $F_s(R)$ is decomposable over 1 coordinate of $R$ at a time, \emph{we can construct an unbiased estimator from top-1 feedback} ($\{\sigma, R(\sigma(1)) \}$). The unbiased estimator is:
\[
{\bf \tilde{z}}= X^{\top} \left( \dfrac{(\exp(s(\sigma(1))) - \exp(R(\sigma(1))))e_{\sigma(1)}}{p(\sigma(1))}\right)
\]
 
{\bf Other Listwise Methods:} As we mentioned before, most listwise convex surrogates will not be suitable for Algorithm~\ref{alg:RTop-kF} with top-k feedback. For example, the class of popular listwise surrogates that are developed from structured prediction perspective \citep{chapelle2007, yue2007} cannot have unbiased estimator of gradients from top-k feedback since they are based on maps from full relevance vectors to full rankings and thus cannot be decomposed over $k=1$ or $2$ coordinates of $R$. It does not appear they have any natural modification to make them amenable to our approach.

\subsubsection{Non-convex Surrogate}
\label{nonconvexsurrogate}
We provide an example of a non-convex surrogate for which Alg.~\ref{alg:RTop-kF} is applicable (however it will not have any regret guarantees due to non-convexity). We choose the SmoothDCG surrogate given in \citep{chapelle2010gradient}, which has been shown to have very competitive empirical performance. SmoothDCG, like ListNet, defines a family of surrogates, based on the cut-off point of DCG (see original paper \citep{chapelle2010gradient} for details). We consider  SmoothDCG@1, which is the smooth version of DCG@1 (i.e., DCG which focuses just on the top-ranked document). 
The surrogate is defined as: $\phi_{SD}(s,R)= \frac{1}{\sum_{j=1}^m \exp(s(j)/{\epsilon})} \sum_{i=1}^m G(R(i)) \exp(s(i)/{\epsilon})$, where $\epsilon$ is a (known) smoothing parameter and $G(a)= 2^a-1$. The gradient of the surrogate is:
\[
\begin{split}
&[\nabla_s \sdcg(s,R)] = \sum_{i=1}^m h_{s,i}(R(i)), \\
& h_{s,i}(R(i))=  G(R_i) \left(\sum_{j=1}^m \frac{1}{\epsilon} [ \frac{\exp(s(i)/\epsilon)}{\sum_{j'} \exp(s(j')/\epsilon)} \mathbbm{1}_{(i = j)} 
- \frac{\exp((s(i)+s(j))/\epsilon)}{(\sum_{j'} \exp(s(j')/\epsilon))^2} ]e_j \right)
\end{split}
\]
Using Lemma~\ref{unbiasedestimator}, we can write $F(R)= F_s(R)= \sum_{i=1}^m h_{s,i}(R(i))$ where $h_{s,i}(R(i))$ is defined above. Since $F_s(R)$ is decomposable over 1 coordinate of $R$ at a time, \emph{we can construct an unbiased estimator from top-1 feedback} ($\{\sigma, R(\sigma(1)) \}$), with unbiased estimator being:
\[
\begin{split}s(\sigma(1))
\tilde{z} =  X^{\top} \left(\dfrac{G(R(\sigma(1)))}{p(\sigma(1))} \sum_{j=1}^m \frac{1}{\epsilon} [ \dfrac{\exp(s(\sigma(1))/\epsilon)}{\sum_{j'} \exp(s(j')/\epsilon)} \mathbbm{1}_{(\sigma(1) = j)}
- \dfrac{\exp((s(\sigma(1))+s(j))/\epsilon)}{(\sum_{j'} \exp(s(j')/\epsilon))^2} ]e_j 
 \right)\\
\end{split}
\]

\subsection {Computational Complexity of Algorithm~\ref{alg:RTop-kF}}
\label{complexity-C}
Three of the four key steps governing the complexity of Algorithm~\ref{alg:RTop-kF}, i.e., construction of $\tilde{s}_t$, $\tilde{\sigma}_t$ and sorting can all be done in $O(m \log(m))$ time.The only bottleneck could have been calculations of $p(\tilde{\sigma_t}(1))$ in squared loss, (modified) ListNet loss and SmoothDCG loss, and $p(\tilde{\sigma_t}(1),\tilde{\sigma_t}(2))$ in RankSVM loss, since they involve sum over permutations. However, they have a compact representation, i.e., $p(\tilde{\sigma_t}(1))= (1 -\gamma + \frac{\gamma}{m}) \mathbbm{1}(\tilde{\sigma_t}(1)= \sigma_t(1)) + \frac{\gamma}{m} \mathbbm{1}(\tilde{\sigma_t}(1)\neq \sigma_t(1)) $ and $p(\tilde{\sigma_t}(1),\tilde{\sigma_t}(2))= (1 -\gamma +\frac{\gamma}{m(m-1)})\mathbbm{1}(\tilde{\sigma_t}(1)= \sigma_t(1), \tilde{\sigma_t}(2)= \sigma_t(2)) + \frac{\gamma}{m(m-1)} [\sim \mathbbm{1}(\tilde{\sigma_t}(1)= \sigma_t(1), \tilde{\sigma_t}(2)= \sigma_t(2))] $. The calculations follow easily due to the nature of $\mathbb{P}_t$ (step-6 in algorithm) which put equal weights on all permutations other than $\sigma_t$.

\subsection{Regret Bounds}
The underlying deterministic part of our algorithm is online gradient descent (OGD) \citep{zinkevich2003online}. The regret of OGD, run with unbiased estimator of gradient of a {\bf convex} function, as given in Theorem 3.1 of \citep{flaxman2005}, in our problem setting is:
\begin{equation}
\label{eq:flaxman}
\begin{split}
\E \left[\sum_ {t=1}^T \phi(X_tw_t, R_t) \right] \le \underset{w:\|w\|_2 \le U} {\min} \sum_{t=1}^T \phi(X_t w,R_t) + \frac{U^2}{2\eta} +  \frac{\eta}{2} \E \left[ {\sum_{t=1}^T \|\tilde{z}_t\|_2^2} \right]
\end{split} 
\end{equation}
where $\tilde{z}_t$ is unbiased estimator of $\nabla_{w=w_t}\phi(X_t w,R_t)$,  conditioned on past events, $\eta$ is the learning rate and the expectation is taken over all randomness in the algorithm.

However, from the perspective of the loss $\phi(\tilde{s}_t,R_t)$ incurred by Algorithm~\ref{alg:RTop-kF}, at each round $t$, the RHS above is not a valid upper bound. The algorithms plays the score vector suggested by OGD ($\tilde{s}_t= X_t w_t$) with probability $1-\gamma$ (exploitation) and plays a randomly selected score vector (i.e., a draw from the uniform distribution on $[0,1]^m$), with probability $\gamma$ (exploration). Thus, the expected number of rounds in which the algorithm does not follow the score suggested by OGD is $\gamma T$, leading to an extra regret of order $\gamma T$. Thus, we have \footnote{The instantaneous loss suffered at each of the exploration round can be maximum of $O(1)$, as long as $\phi(s,R)$ is bounded, $\forall \ s$ and $\forall \ R$. This is true because the score space is $\ell_2$ norm bounded, maximum relevance grade is finite in practice and we consider Lipschitz, convex surrogates.}
\begin{equation}
\label{eq:exploration}
\E \left[\sum_ {t=1}^T \phi(\tilde{s}_t, R_t) \right] \le \E \left[\sum_ {t=1}^T \phi(X_t w_t, R_t) \right]  + O \left(\gamma T \right) 
\end{equation}

We first control $\E_t \|\tilde{z}_t\|_2^2$, for all convex surrogates considered in our problem (we remind that $\tilde{z}_t$ is the estimator of a gradient of a surrogate, calculated at time $t$. In Sec~\ref{convexsurrogates} , we omitted showing $w$ in $s^w$ and index $t$).
To get bound on  $\E_t \|\tilde{z}_t\|_2^2$, we used the following norm relation that holds for any matrix $X$ \citep{bhaskara2011}: $\|X\|_{p \to q}= \underset{v \neq 0}{\sup} \frac{\|Xv\|_q}{\|v\|_p}$, where $q$ is the dual exponent of $p$ (i.e., $\tfrac{1}{q}+\tfrac{1}{p} = 1$), and the following lemma derived from it:
\begin{lem}
\label{normexpr}
For any $1 \leq p \leq \infty$, $\| X^{\top} \|_{1 \to p} = \| X \|_{q \to \infty} = \max_{j=1}^m \| X_{j:} \|_p$, 
where $X_{j:}$ denotes $j$th row of $X$ and $m$ is the number of rows of matrix.
\end{lem}

We have the following result:
\begin{lem}
\label{expectednorm}
For parameter $\gamma$ in Algorithm~\ref{alg:RTop-kF} , $R_D$ being the bound on $\ell_2$ norm of the feature vectors (rows of document matrix $X$), $m$ being the upper bound on number of documents per query, $U$ being the radius of the Euclidean ball denoting the space of ranking parameters and $R_{\max}$ being the maximum possible relevance value (in practice always $\le$ 5), let $C^{\phi} \in \{C^{sq}, C^{svm}, C^{KL}\}$ be polynomial functions of $R_D, m, U, R_{max}$, where the degrees of the polynomials depend on the surrogate ($\phi_{sq}, \phi_{svm}, \phi_{KL}$), with no degree ever greater than four. Then we have, 
\begin{equation}
\label{eq:expectednorm}
\begin{split}
 \E_t \left[\|\tilde{z}_t\|_2^2 \right] \le \dfrac{C^{\phi}}{\gamma}
\end{split}
\end{equation}
\end{lem}

Plugging Eq.~\ref{eq:expectednorm} and Eq.~\ref{eq:exploration} in Eq.~\ref{eq:flaxman}, and optimizing over $\eta$ and $\gamma$, (which gives $\eta= O(T^{-2/3})$ and $\gamma= O(T^{-1/3})$), we get the final regret bound:

\begin{thm}
\label{theoryboundinpartial}
For any sequence of instances and labels $(X_t,R_t)_{\{t \in [T]\}}$, applying Algorithm~\ref{alg:RTop-kF} with top-1 feedback for $\phi_{sq}$ and $\phi_{KL}$ and top-2 feedback for $\phi_{svm}$,  will produce the following bound:
\begin{equation}
\begin{split}
\E \left[\sum_{t=1}^T \phi(\tilde{s}_t, R_t) \right] -  \underset{w:\|w\|_2 \le U}{\min} \sum_{t=1}^T \phi(X_t w,R_t)  \le C^{\phi} O \left(T^{2/3} \right)
\end{split}
\end{equation}
where $C^{\phi}$ is a surrogate dependent function, as described in Lemma~\ref{expectednorm} , and  expectation is taken over underlying randomness of the algorithm, over $T$ rounds.
\end{thm}

{\bf Discussion:} It is known that online bandit games are special instances of partial monitoring games. For bandit online convex optimization problems with Lipschitz, convex surrogates, the best regret rate known so far, that can be achieved by an efficient algorithm, is $O(T^{3/4})$ (however, see the work of \cite{bubeck2015multi} for a non-constructive $O(\log^4(T) \sqrt{T})$ bound). Surprisingly, Alg.~\ref{alg:RTop-kF}, when applied in a partial monitoring setting to the Lipschitz, convex surrogates that we have listed, achieves a better regret rate than what is known in the bandit setting. Moreover, as we show subsequently, for an entire class of Lipschitz convex surrogates (subclass of NDCG calibrated surrogates), sub-linear (in $T$) regret is not even achievable. Thus, our work indicates that even within the class of Lipschitz, convex surrogates, regret rate achievable is dependent on the structure of surrogates; something that does not arise in bandit convex optimization. 

\subsection{Impossibility of Sublinear Regret  for NDCG Calibrated Surrogates}
Learning to rank methods optimize surrogates to learn a ranking function, even though performance is measured by target measures like NDCG. This is done because direct optimization of the measures lead to NP-hard optimization problems. One of the most desirable properties of any surrogate is \emph{calibration}, i.e., the surrogate should be calibrated w.r.t the target \citep{bartlett2006convexity}. Intuitively, it means that a function with small expected surrogate loss on unseen data should have small expect target loss on unseen data.  We focus on NDCG calibrated surrogates (both convex and non-convex) that have been characterized by \cite{ravikumar2011ndcg}. We first state the necessary and sufficient condition for a surrogate to be calibrated w.r.t NDCG. For any score vector $s$ and distribution $\eta$ on relevance space $\mathcal{Y}$, let $\bar{\phi}(s,\eta)= \E_{R \sim \eta} \phi(s,R)$. Moreover, we define $G({\bf R})= (G(R_1),\ldots,G(R_m))^{\top}$. $Z(R)$ is defined in Sec~\ref{rankingmeasures-NC}.
\begin{thm}
\label{calibration}
\cite[Thm. 6]{ravikumar2011ndcg} A surrogate $\phi$ is NDCG calibrated iff for any distribution $\eta$ on relevance space $\mathcal{Y}$, there exists an invertible, order preserving map $g: \reals^m \mapsto \reals^m$ s.t. the unique minimizer $s^{*}_{\phi}(\eta)$ can be written as
\begin{equation}
\label{eq:calibration}
s^*_{\phi}(\eta)= g \left( \E_{R \sim \eta} \left[\frac{G({\bf R})}{Z(R)}\right] \right) .
\end{equation}
\end{thm}
Informally, Eq.~\ref{eq:calibration} states that $\argsort(s^*_{\phi}(\eta)) \subseteq \argsort(\E_{R \sim \eta} \left[\tfrac{G({\bf R})}{Z(R)}\right])$ 
\cite{ravikumar2011ndcg} give concrete examples of NDCG calibrated surrogates, including how some of the popular surrogates can be converted into NDCG calibrated ones: e.g., the NDCG calibrated version of squared loss is $\|s- \frac{G({\bf R})}{Z(R)}\|_2^2$.  

We now state the \emph{impossibility} result for the class of NDCG calibrated surrogates when feedback is restricted to top ranked item.

\begin{thm}
\label{impossiblegame}
Fix the online learning to rank game with top $1$ feedback and any NDCG calibrated surrogate. Then, for every learner's algorithm, there exists an adversary strategy such that the learner's expected regret is $\Omega(T)$. 
\end{thm}

We note that the proof of Theorem.\ 3 of \cite{piccolboni2001discrete} cannot be directly extended to prove the impossibility result because it relies on constructing a connected graph on vertices defined by neighboring actions of learner. In our case, due to the continuous nature of learner's actions, the graph will be an empty graph and proof will break down.

%% file: Experiments.tex
\section{Experiments}
\label{experiments}
We conducted experiments on simulated and commercial datasets to demonstrate the performance of our algorithms.
\subsection{Non Contextual Setting}
\label{experiments-NC}
{\bf Objectives}: We had the following objectives while conducting experiments in the non-contextual, online ranking with partial feedback setting:
\begin{itemize}
\item Investigate how performance of Algorithm~\ref{alg:top-k} is affected by size of blocks during blocking.
\item Investigate how performance of the algorithm is affected by amount of feedback received (i.e., generalizing $k$ in top $k$ feedback).
\item Demonstrate the difference between regret rate of our algorithm, which operates in partial feedback setting, with regret rate of a full information algorithm which receives full relevance vector feedback at end of each round.
\end{itemize}
We applied Algorithm~\ref{alg:top-k} in conjunction with Follow-The-Perturbed-Leader (FTPL) full information algorithm, as described in Sec.~\ref{algorithm-NC}. We note that since our work is first of its kind in the literature, we had no comparable baselines. The generic partial monitoring algorithms that do exist cannot be applied due to computational inefficiency (Sec.~\ref{minimax-NC}).\\\\
{\bf Experimental Setting}: All our experiments were conducted with respect to the DCG measure, which is quite popular in practice, and binary graded relevance vectors. Our experiments were conducted on the following simulated dataset. We fixed number of items to $20$( $m=20$). We then fixed a ``true" relevance vector which had $5$ items with relevance level $1$ and $15$ items with relevance level $0$. We then created a total of T=$10000$ relevance vectors by corrupting the true relevance vector. The corrupted copies were created by independently flipping  each relevance level ($0$ to $1$ and vice-versa) with a small probability. The reason for creating adversarial relevance vectors in such a way was to reflect diversity of preferences in practice. In reality, it is likely that most users will have certain similarity in preferences, with small deviations on certain items and certain users. That is, some items are likely to be relevance in general, with most items not relevant to majority of users, with slight deviation from user to user. Moreover, the comparator term in our regret bound (i.e., cumulative loss/gain of the true best ranking in hindsight) only makes sense if there is a ranking which satisfies most users.\\\\
{\bf Results}: We plotted average regret over time under DCG. Average regret over time means cumulative regret up to time $t$, divided by $t$, for $1 \le t \le T$. Figure~\ref{Fig1} demonstrates the effect of block size on the regret rate under DCG. We fixed feedback to relevance of top ranked item ($k=1$). As can be seen in Corollary~\ref{efficientregret-DCG}, the optimal regret rate is achieved by optimizing over number of blocks (hence block size), which requires prior knowledge of time horizon $T$. We wanted to demonstrate how the regret rate (and hence the performance of the algorithm) differs with different block sizes. The optimal number of blocks in our setting is $K \sim 200$, with corresponding block size being $\ceil{T/K}= 50$. As can be clearly seen, with optimal block size, the regret drops fastest and becomes steady after a point. $K=10$ means that block size is $1000$. This means over the time horizon, number of exploitation rounds greatly dominates number of exploration rounds, leading to regret dropping at a slower rate initially than than the case with optimal block size. However, the regret drops of pretty sharply later on. This is because the relevance vectors are slightly corrupted copies of a ``true" relevance vector and the algorithm gets a good estimate of the true relevance vector quickly and then more exploitation helps. When $K=400$ (i.e, block size is $25$), most of the time, the algorithm is exploring, leading to a substantially worse regret and poor performance.

Figure~\ref{Fig2} demonstrates the effect of amount of feedback on the regret rate under DCG. We fixed $K=200$, and varied feedback as relevance of top $k$ ranked items per round, where $k=1,5,10$. Validating our regret bound, we see that as $k$ increases, the regret decreases.

Figure~\ref{Fig3} compares regret of our algorithm, working with top $1$ feedback and FTPL full information algorithm, working with full relevance vector feedback at end of each round. We fixed $K=200$ and the comparison was done from $1000$ iterations onwards, i.e., roughly after the initial learning phase. FTPL full information algorithm has regret rate of $O(T^{1/2})$ (ignoring other parameters). So, as expected, FTPL with full information feedback outperforms our algorithm with highly restricted feedback; yet, we have demonstrated, both theoretically and empirically, that it is possible to have a good ranking strategy with highly restricted feedback.

\begin{figure}[h]
\begin{center}
\centerline{\includegraphics[height=70mm, width=130mm]{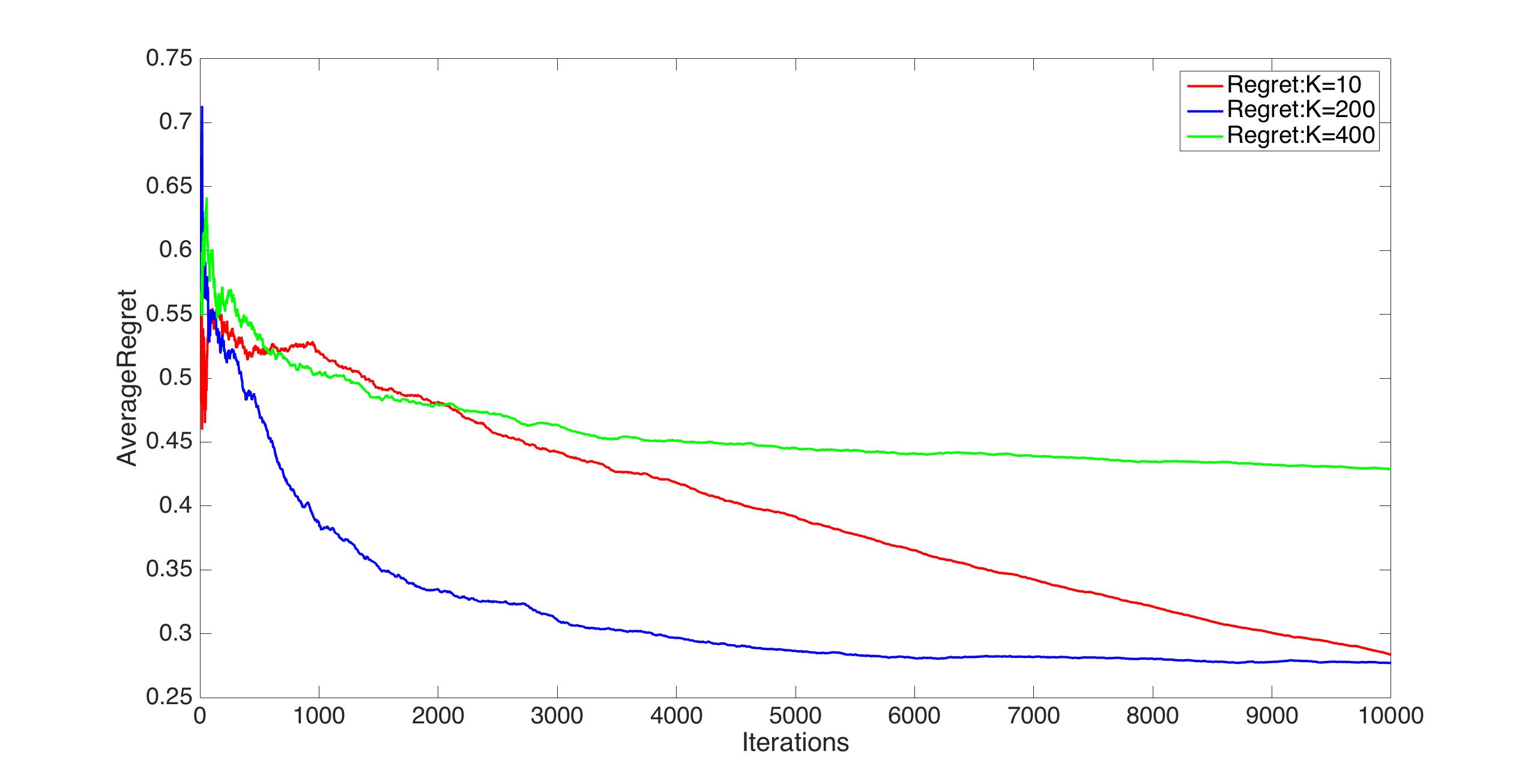}}
\caption{Average regret under DCG, with feedback on top ranked object, for varying block size, where block size is $\ceil{T/K}$. \emph{Best viewed in color}. } \label{Fig1}
\end{center}
\end{figure} 

\begin{figure}[h]
\begin{center}
\centerline{\includegraphics[height=70mm, width=130mm]{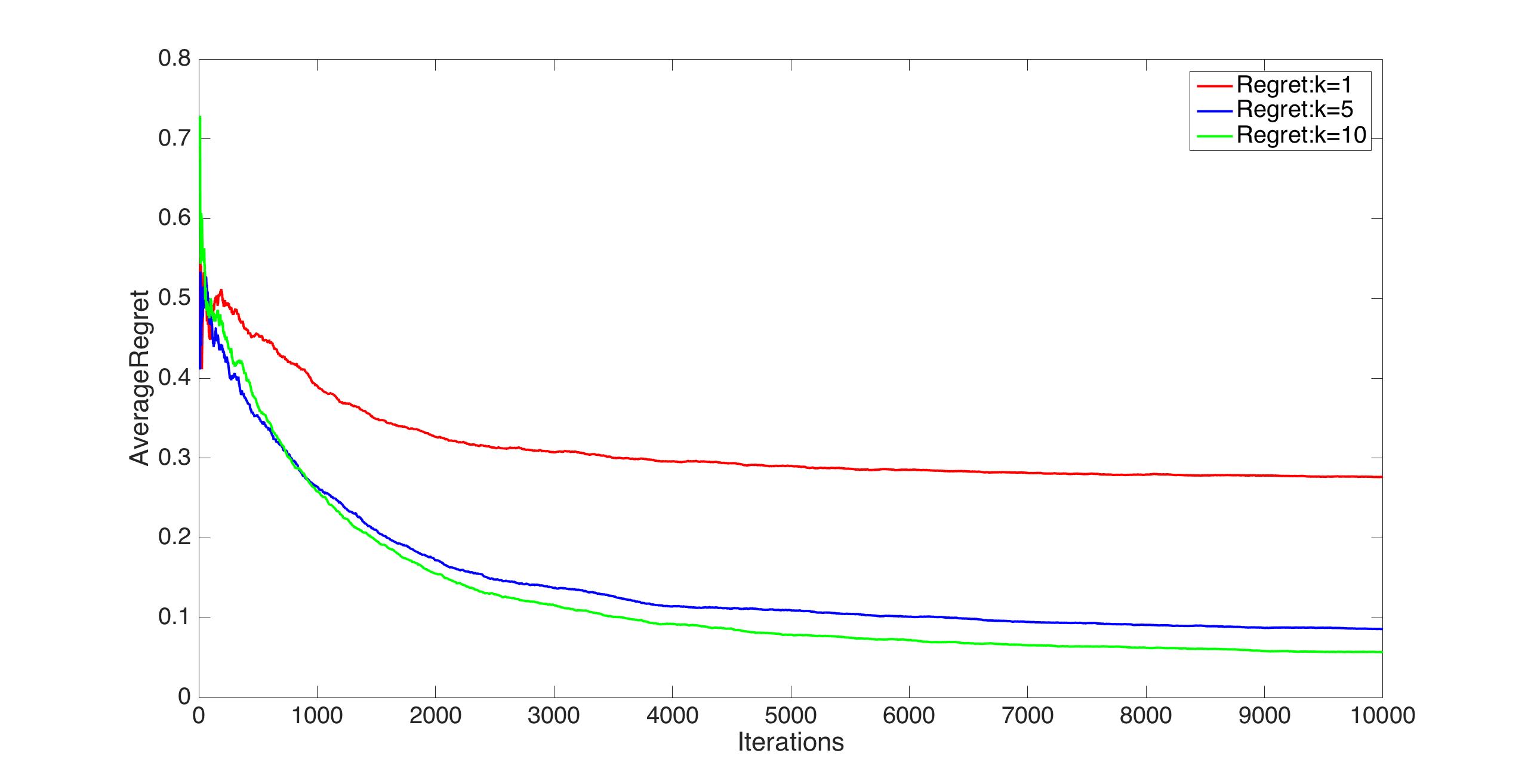}}
\caption{Average regret under DCG, where $K=200$, for varying amount of feedback.  \emph{Best viewed in color. }}\label{Fig2}
\end{center}
\end{figure} 

\begin{figure}[h]
\begin{center}
\centerline{\includegraphics[height=70mm, width=130mm]{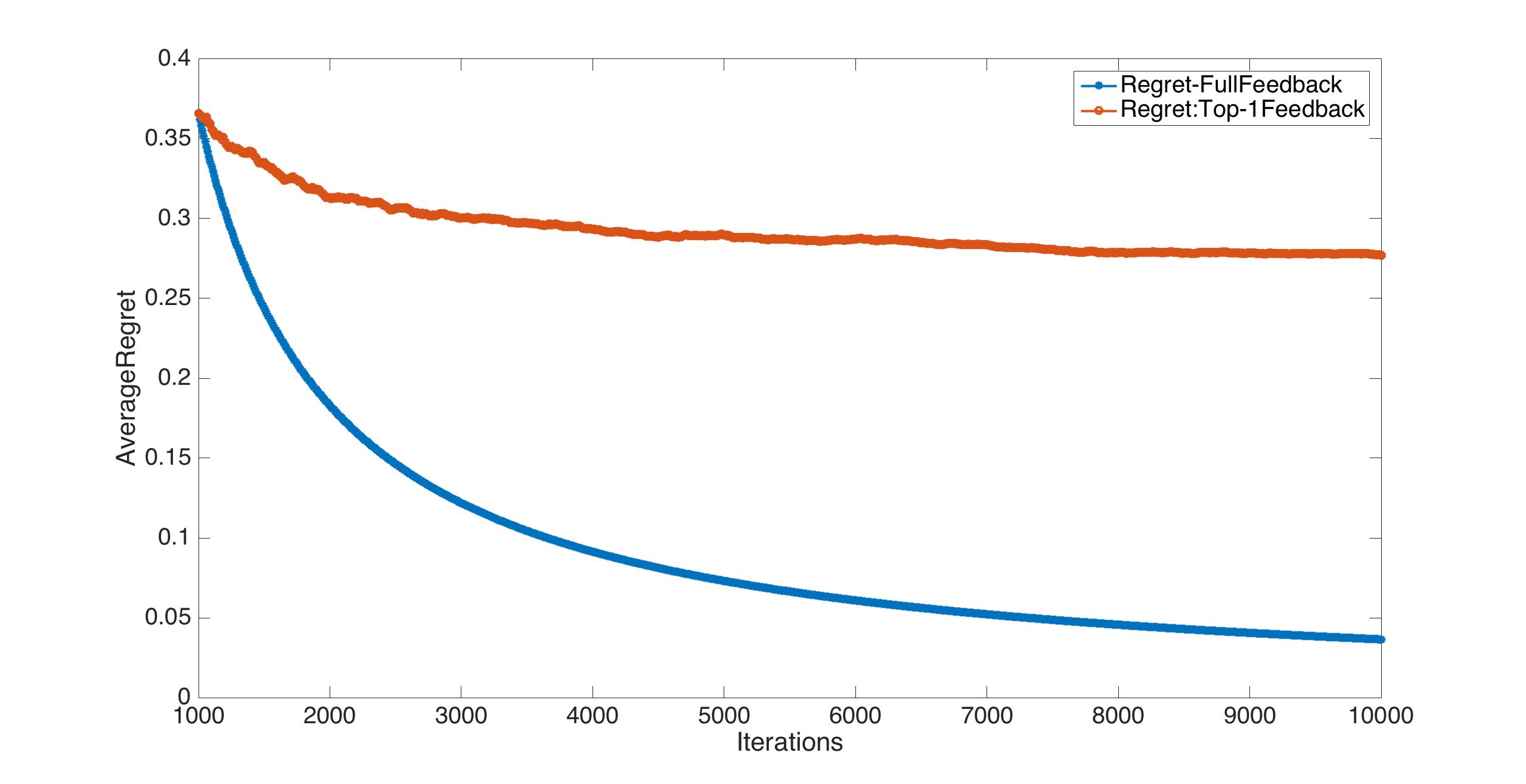}}
\caption{Comparison of average regret over time, for DCG, between top-1 feedback and full relevance vector feedback.  \emph{Best viewed in color. }}\label{Fig3}
\end{center}
\end{figure} 

\subsection{Contextual Setting}
\label{experiments-C}

\begin{figure}[h]
\begin{center}
\centerline{\includegraphics[height=70mm, width=130mm]{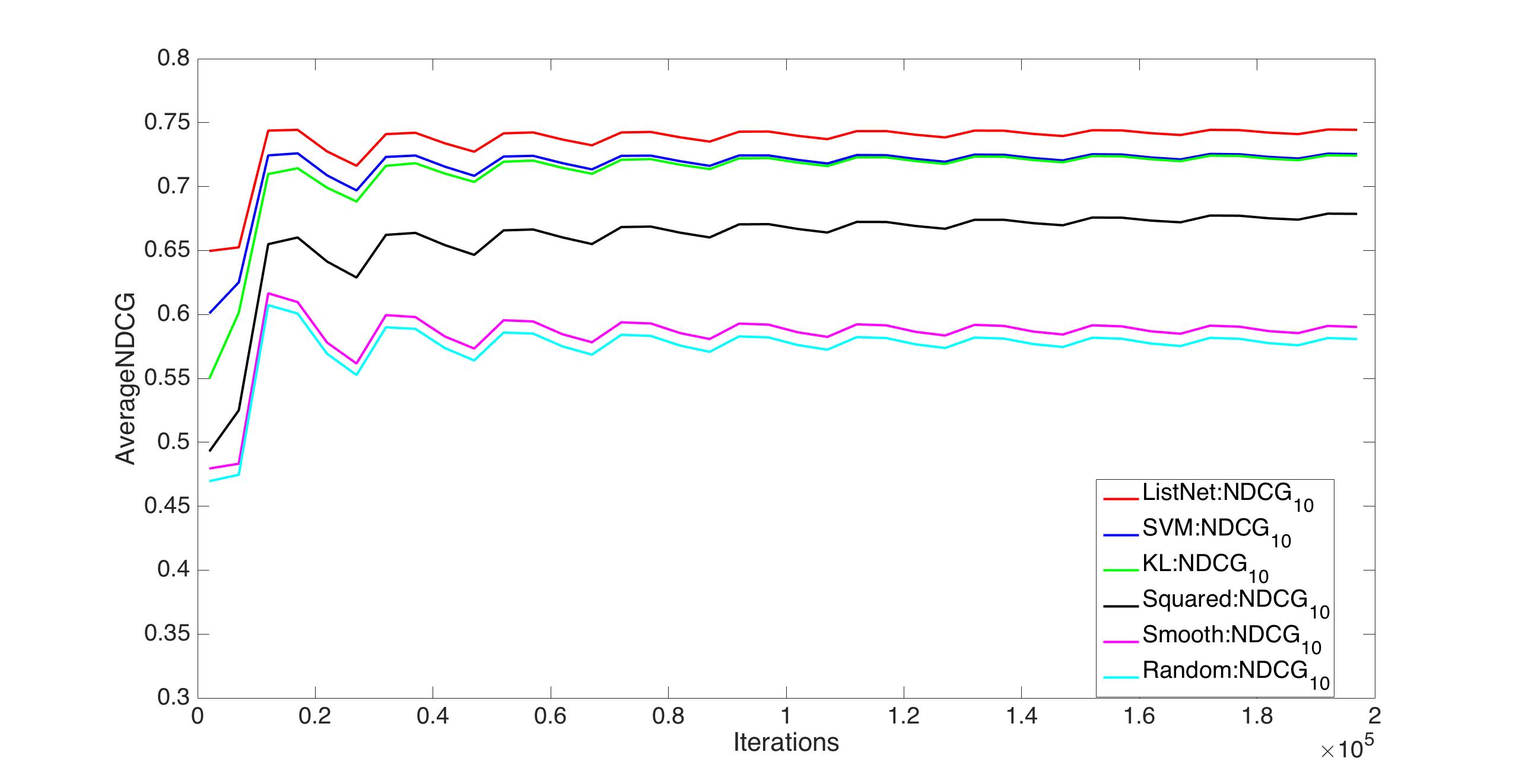}}
\caption{Average $\text{NDCG}_{10}$ values of different algorithms, for Yahoo dataset. ListNet:$\text{NDCG}_{10}$ (in cyan) operates on full feedback and Random:$\text{NDCG}_{10}$ (in red) does not receive any feedback. \emph{Best viewed in color }.}\label{Fig4}
\end{center}
\end{figure} 

\begin{figure}[h]
\begin{center}
\centerline{\includegraphics[height=70mm, width=130mm]{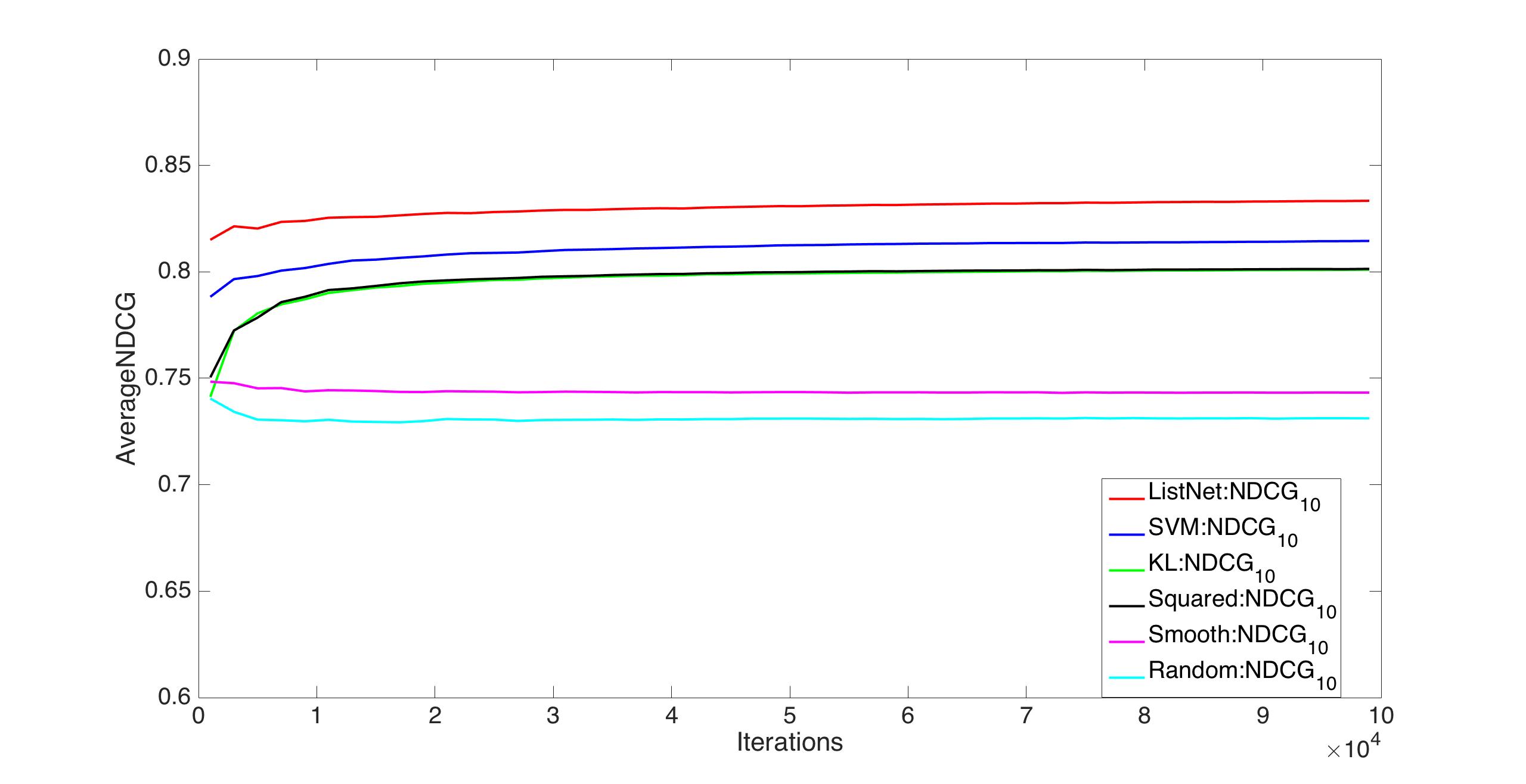}}
\caption{Average $\text{NDCG}_{10}$ values of different algorithms, for Yahoo dataset. ListNet:$\text{NDCG}_{10}$ (in cyan) operates on full feedback and Random:$\text{NDCG}_{10}$ (in red) does not receive any feedback. \emph{Best viewed in color. }}\label{Fig5}
\end{center}
\end{figure}

{\bf Objective:}  Since our contextual, online learning to rank with restricted feedback setting involves  query-document matrices, we could conduct experiments on commercial, publicly available ranking datasets.Our objective was to demonstrate that it is possible to learn a good ranking function, even with highly restricted feedback, when standard online learning to rank algorithms would require full feedback at end of each round. As stated before, though our algorithm is designed to minimize surrogate based regret,  the surrogate loss is only of interest. The users only care about the ranking presented to them, and indeed the algorithm interacts with users by presenting ranked lists and getting feedback on top ranked item(s). We tested the quality of the ranked lists, and hence the performance of the evolving ranking functions, against the full relevance vectors, via ranking measure NDCG, cutoff at the $10$th item. NDCG, cutoff at a point $n$, is defined as follows:
\[
NDCG_n(\sigma,R)= \frac{1}{Z_n(R)}\sum_{i=1}^n \frac{2^{R(\sigma(i))}-1}{\log_2(1+i)}
\]
where $Z_n(R)= \underset{\sigma \in S_m}{\max} \sum_{i=1}^n  \frac{2^{R(\sigma(i))}-1}{\log_2(1+ i)}$. We want to emphasize that the algorithm was in no way affected by the fact that we were measuring its performance with respect to NDCG cutoff at $10$. In fact, the cutoff point can be varied, but usually, researchers report performance under NDCG cutoff at $5$ or $10$.  \\\\
{\bf Baselines:} We applied Algorithm~\ref{alg:RTop-kF}, with top $1$ feedback, on Squared, KL (un-normalized ListNet) and SmoothDCG surrogates, and with top $2$ feedback, on the RankSVM surrogate. Based on the objective of our work, we selected two different ranking algorithms as baselines. The first one is the online version of ListNet ranking algorithm (which is essentially OGD on cross-entropy function), with full relevance vector revealed at end of every round. ListNet is not only one of the most cited ranking algorithms (over 700 citations according to Google Scholar), but also one of the most validated algorithms \citep{tax2015}. We emphasize that some of the ranking algorithms in literature, which have shown better empirical performance than ListNet, are based on non-convex surrogates with complex, non-linear ranking functions. These algorithms cannot usually be converted into online algorithms which learn from streaming data. Our second algorithm is a simple, fully random algorithm , which outputs completely random ranking of documents at each round. This algorithm, in effect, receives no feedback at end of each round. \emph{Thus, we are comparing Algorithm~\ref{alg:RTop-kF}, which learns from highly restricted feedback, with an algorithm which learns from full feedback and an algorithm which receives no feedback and learns nothing.}  \\\\
{\bf Datasets:} We compared the various ranking functions on two large scale commercial datasets. They were Yahoo's Learning to Rank Challenge dataset \citep{chapelle2011yahoo} and a dataset published by Russian search engine Yandex \citep{Yandex}. The Yahoo dataset has 19944 unique queries with 5 distinct relevance levels, while Yandex has 9126 unique queries with 5 distinct relevance levels.\\\\
{\bf Experimental Setting}: We selected time horizon $T= 200,000$ (Yahoo) and $T=100,000$ (Yandex) iterations for our experiments (thus, each algorithm went over each dataset multiple times). The reason for choosing different time horizons is that there are roughly double the number of queries in Yahoo dataset as compared to Yandex dataset. All the online algorithms, other than the fully random one, involve learning rate $\eta$ and exploration parameter $\gamma$ (full information ListNet does not involve $\gamma$ and SmoothDCG has an additional smoothing parameter $\epsilon$). While obtaining our regret guarantees, we had established that $\eta=O(T^{-2/3})$ and $\gamma= O(T^{-1/3})$. In our experiments, for each instance of Algorithm~\ref{alg:RTop-kF}, we selected a time varying $\eta=\frac{0.01}{t^{2/3}}$ and $\gamma=\frac{0.1}{t^{1/3}}$, for round $t$ . We fixed $\epsilon=0.01$. For ListNet, we selected $\eta=\frac{0.01}{t^{1/2}}$, since regret guarantee in OGD is established with $\eta=O(T^{-1/2})$. We plotted average $\text{NDCG}_{10}$ against time, where average  $\text{NDCG}_{10}$ at time $t$ is the cumulative $\text{NDCG}_{10}$ up to time $t$, divided by $t$. 
We made an important observation while comparing the performance plots of the algorithms. As we have shown, construction of the unbiased estimators involve division by a probability value (Eq~\ref{eq:unbiasedestimator}). The particular probability value can be $\frac{\gamma}{m}$, which is very small since $\gamma$ goes to $0$, when the top ranked item of the randomly drawn permutation does not match the top ranked item of the permutation given by the deterministic score vector (Sec~\ref{complexity-C}). The mismatch happens with very low probability (since the random permutation is actually the deterministic permutation with high probability).  While theoretically useful, in practice, dividing by such small value negatively affected the gradient estimation and hence the performance of our algorithm. So, when the mismatch happened, we scaled up $\gamma$ on the mismatch round by a constant, to remove the negative effect. \\\\
{\bf Results:} Figure~\ref{Fig4} and Figure~\ref{Fig5} show that ListNet, with full information feedback at end of each round, has highest average NDCG value throughout, as expected. However, Algorithm~\ref{alg:RTop-kF}, with the convex surrogates, produce competitive performance. In fact, in the Yahoo dataset, our algorithm, with RankSVM and KL, are very close to the performance of ListNet. RanSVM based algorithm does better than the others, since the estimator of RankSVM gradient is constructed from top $2$ feedback, leading to lower variance of the estimator. KL based algorithm does much better than Squared loss based algorithm on Yahoo and equally as well on Yandex dataset.  \emph{Crucially, our algorithm, based on all three convex surrogates, perform significantly better than the purely random algorithm, and are much closer to ListNet in performance, despite being much closer to the purely random algorithm in terms of feedback}. Our algorithm, with SmoothDCG, on the other hand, produce poor performance. We believe the reason is the non-convexity of the surrogate, which leads to the optimization procedure possibly getting stuck at a local minima. In batch setting, such problem is avoided by an annealing technique that successively reduces $\epsilon$. We are not aware of an analogue in an online setting. Possible algorithms optimizing non-convex surrogates in an online manner, which require gradient of the surrogate, may be adapted to this partial feedback setting. The main purpose for including SmoothDCG in our work was to show that unbiased estimation of gradient, from restricted feedback, is possible even for non-convex surrogates.

%% file: QuestionsandConclusion.tex
\section{Conclusion and Future Directions}
\label{conclusion}
We studied the problem of online learning to rank with a novel, restricted feedback model. The work is divided into two parts: in the first part, the set of items to be ranked is fixed, with varying user preferences, and in the second part, the items vary, as traditional query-documents matrices. The parts are tied by the feedback model; where the user gives feedback only on top $k$ ranked items at end of each round, though the performance of the learner's ranking strategy is judged against full, \emph{implicit} relevance vectors. In the first part, we gave comprehensive results on learnability with respect to a number of practically important ranking measures. We also gave a generic algorithm, along with an efficient instantiation, which achieves sub-linear regret rate for certain ranking measures. In the second part, we gave an efficient algorithm, which works on a number of popular ranking surrogates, to achieve sub-linear regret rate. We also gave an impossibility result for an entire class of ranking surrogates. Finally, we conducted experiments on simulated and commercial ranking datasets to demonstrate the performance of our algorithms.

We highlight some of the open questions of interest:
\begin{itemize}
\item What are the minimax regret rates with top $k$ feedback model, for $k>1$, for the ranking measures DCG, PairwiseLoss, Precision@$n$ and their normalized versions NDCG, AUC and AP? Specifically, NDCG and AP are very popular in the learning to rank community. We showed that with top $1$ feedback model, no algorithm can achieve sub-linear regret for NDCG and AP. Is it possible to get sub-linear regret with $1<k < m$?
\item We used FTPL as the sub-routine in Algorithm~\ref{alg:top-k} to get an efficient algorithm. It might be possible to use other full information algorithms as sub-routine, retaining the efficiency, but getting tighter rates in terms of parameters (other than $T$) and better empirical performance.
\item We applied Algorithm~\ref{alg:RTop-kF} on three convex surrogates and one non-convex surrogates. It would be interesting to investigate what other surrogates the algorithm can be applied on, guided by Lemma~\ref{unbiasedestimator}, and test its empirical performance. Since the algorithm learns a ranking function in the traditional query-documents setting, the question is more of practical interest.
\item We saw that Algorithm~\ref{alg:RTop-kF}, when applied to SmoothDCG, does not produce competitive empirical performance. It has been shown that a ranking function, learnt by optimizing SmoothDCG in the batch setting, has extremely competitive empirical performance \citep{letor}. In the batch setting, simulated annealing is used to prevent the optimization procedure getting stuck in local minima. Any algorithm that optimizes non-convex surrogates in an online manner, by accessing its gradient, can replace the online gradient descent part in our algorithm and tested on SmoothDCG for empirical performance.
\item We proved an impossibility result for NDCG calibrated surrogates with top $1$ feedback. What is the minimax regret for NDCG calibrated surrogates, with top $k$ feedback, for $k>1$?  
\end{itemize}

%% file: Appendix-A.tex
{\bf Proof of Theorem~\ref{localobservability}}:
\begin{proof}
We will explicitly show that local observability condition fails by considering the case when number of objects is $m=3$. Specifically, action pair \{$\sigma_1, \ \sigma_2$\}, in Table \ref{lossmatrix-table} are neighboring actions, using Lemma \ref{neighbor-actions} . Now every other action $\{\sigma_3,\sigma_4,\sigma_5,\sigma_6\}$ either places object 2 at top or object 3 at top. It is obvious that the set of probabilities for which $E[R(1)]\ge E[R(2)]=E[R(3)]$ cannot be a subset of any $C_3,C_4,C_5, C_6$. From Def. 4, the neighborhood action set of actions $\{\sigma_1, \sigma_2\}$ is precisely $\sigma_1$ and $\sigma_2$ and contains no other actions.
By definition of signal matrices $S_{\sigma_1}, \ S_{\sigma_2}$ and entries $\ell_1, \ \ell_2$ in Table \ref{lossmatrix-table} and \ref{feedbackmatrix-table}, we have,
\begin{equation}
\begin{aligned}
 S_{\sigma_1}= S_{\sigma_2}=
  \left[ {\begin{array}{cccccccc}
   1 & 1 & 1 & 1 & 0 & 0 & 0 & 0 \\
   0 & 0 & 0 & 0 & 1 & 1 & 1 & 1 \\
  \end{array} } \right]\\
\ell_1 - \ell_2 = 
 \left[ {\begin{array}{cccccccc}
   0 & 1 & -1 & 0 & 0 & 1 & -1 & 0 \\
  \end{array} } \right] .
\end{aligned}
\end{equation}
It is clear that $\ell_1 - \ell_2 \notin Col(S^{\top}_{\sigma_1})$. Hence, Definition 5 fails to hold.
\end{proof}

{\bf Proof of Theorem~\ref{regret}}:

\begin{proof}

{\bf Full information feedback}: Instead of top $k$ feedback, assume that at end of each round, after learner reveals its action, the full relevance vector $R$ is revealed to the learner. Since the knowledge of full relevance vector allows the learner to calculate the loss for every action (SumLoss$(\sigma,R)$, $\forall \ \sigma$), the game is in full information setting, and the learner, using the full information algorithm, will have an $O(C \sqrt{T})$ expected regret for $SumLoss$ (ignoring computational complexity). Here, $C$  denotes parameter specific to the full information algorithm used. 

{\bf Blocking with full information feedback}: We consider a blocked variant of the full information algorithm. We still assume that full relevance vector is revealed at end of each round. Let the time horizon $T$ be divided into $K$ blocks, i.e., $\{B_1,\ldots,B_K\}$, of equal size. Here, $B_i= \{(i-1)(T/K)+1, (i-1)(T/K)+2, (i-1)T/K +3, \ldots, i (T/K)\}$. While operating in a block, the relevance vectors revealed at end of each round are accumulated, but not used to generate learner's actions like in the ``without blocking" variant. Assume at the start of block $B_i$, there was some vector $s_{i-1} \in \mathbb{R}^m$. Then, at each round in the block, the randomized full information algorithm exploits $s_{i-1}$ and outputs a permutation (basically maintains a distribution over actions, using $s_{i-1}$, and samples from the distribution). At the end of a block, the average of the accumulated relevance vectors ($R_i^{avg}$) for the block is used to update, as $s_{i-1} + R_i^{avg}$, to get $s_{i}$ for the next block. The process is repeated for each block. 

Formally, the full information algorithm creates distribution $\rho_i$ over the actions, at beginning of block $B_i$, exploiting information $s_{i-1}$. Thus, $\rho_i \in \Delta$, where $\Delta$ is the probability simplex over $m!$ actions. Note that $\rho_i$ is a deterministic function of $\{R^{avg}_1, \ldots, R^{avg}_{i-1}\}$.

Since action $\sigma_t$, for $t \in B_i$, is generated according to distribution $\rho_i$ (we will denote this as $\sigma_t \sim \rho_i$), and in block $i$, distribution $\rho_i$ is fixed, we have
\begin{equation*}
\E_{\sigma_t \sim \rho_i} [ \sum_{t \in [B_i]} SumLoss(\sigma_t,R_t) ]= \sum_{t \in B_i} \rho_i \cdot [SumLoss(\sigma_1,R_t), \ldots, SumLoss(\sigma_{m!},R_t)] .
\end{equation*}

(dot product between 2 vectors of length $m!$).

Thus, the total expected loss of this variant of the full information problem is:
\begin{align}
\label{eq:regret1}
\notag  \E \sum_{t=1}^T  [ SumLoss(\sigma_t,R_t) ] & =   \sum_{i=1}^K \E_{\sigma_t \sim \rho_i}[ \sum_{t \in B_i} SumLoss(\sigma_t,R_t) ]\\
\notag & = \sum_{i=1}^K \sum_{t \in B_i} \rho_i \cdot [SumLoss(\sigma_1,R_t), \ldots, SumLoss(\sigma_{m!},R_t)]\\
\notag & =  \sum_{i=1}^K \sum_{t \in B_i} \rho_i \cdot [\sigma^{-1}_1 \cdot R_t, \ldots, \sigma^{-1}_{m!} \cdot R_t)] \ \text{ By defn. of SumLoss}\\
\notag & = \dfrac{T}{K} \sum_{i=1}^K \rho_i \cdot [\sigma^{-1}_1 \cdot R^{avg}_i, \ldots, \sigma^{-1}_{m!}\cdot R^{avg}_i]\\
\notag &= \dfrac{T}{K} \sum_{i=1}^K \E_{\sigma_i \sim \rho_i}[ SumLoss(\sigma_i,R^{avg}_i) ]\\
& = \dfrac{T}{K} \E_{\sigma_1\sim\rho_1,\ldots,\sigma_K\sim \rho_K} \sum_{i=1}^K SumLoss(\sigma_i,R^{avg}_i) 
\end{align}
where $R^{avg}_i = \sum_{t \in B_i} \dfrac{R_t}{T/K}$.
Note that, at end of every block $i \in [K]$, $\rho_i$ is updated to $\rho_{i+1}$.
By the regret bound of the full information algorithm, for $K$ rounds of full information problem, we have:
\begin{equation}
\label{eq:regret2}
\begin{aligned}
\E_{\sigma_1\sim\rho_1,\ldots,\sigma_K\sim\rho_K} \sum_{i=1}^K SumLoss(\sigma_i,R^{avg}_i)  & \le \underset{\sigma}{\min} \sum_{i=1}^K SumLoss(\sigma ,R^{avg}_i)  + C \sqrt{K}\\
& = \underset{\sigma}{\min} \sum_{i=1}^K \sigma^{-1} \cdot R^{avg}_i + C \sqrt{K}\\
& = \underset{\sigma}{\min} \sum_{t=1}^T \sigma^{-1} \cdot \dfrac{R_t}{T/K} + C \sqrt{K}\\
\end{aligned}
\end{equation}

Now, since
$$\underset{\sigma}{\min}\ {\sum_{t=1}^T \sigma^{-1} \cdot \dfrac{R_t}{T/K}}= \underset{\sigma}{\min}\ \dfrac{1}{T/K} \sum_{t=1}^T SumLoss(\sigma,R_t),$$
combining Eq. \ref{eq:regret1} and Eq. \ref{eq:regret2}, we get:
\begin{equation}
\label{eq:regret3}
\sum_{t=1}^T \E_{\sigma_t \in \rho_i} [ SumLoss(\sigma_t,R_t) ] \le \underset{\sigma}{\min} \sum_{t=1}^T SumLoss(\sigma,R_t)+ C \dfrac{T}{\sqrt{K}}.
\end{equation}

{\bf Blocking with top $k$ feedback}: However, in our top $k$ feedback model, the learner does not get to see the full relevance vector at each end of round.  Thus, we form the unbiased estimator $\hat{R}_i$ of $R^{avg}_i$, using Lemma~\ref{unbiasedestimator}. That is, at start of each block, we choose $\ceil{m/k}$ time points uniformly at random, and at those time points, we output a random permutation which places $k$ distinct objects on top (refer to Algorithm~\ref{alg:top-k}). At the end of the block, we form the vector $\hat{R}_i$ which is the unbiased estimator of $R^{avg}_i$. Note that using random vector $\hat{R}_i$ instead of true $R^{avg}_i$ introduces randomness in the distribution $\rho_i$ itself .
But significantly, $\rho_i$ is dependent only on information received up to the beginning of block $i$ and is independent of the information collected in the block.  We show the exclusive dependence as $\rho_i({\hat{R}_1,\hat{R}_2,..,\hat{R}_{i-1}})$. Thus, for block $i$, we have:
\begin{align*}
& \E_{\sigma_t \sim \rho_i({\hat{R}_1,\hat{R}_2,..,\hat{R}_{i-1}})} \sum_{t \in [B_i]} SumLoss(\sigma_t,R_t)\\
& = \dfrac{T}{K} \E_{\sigma_i \sim \rho_i({\hat{R}_1,\hat{R}_2,..,\hat{R}_{i-1}})} SumLoss(\sigma_i,R^{avg}_i) \\
&\quad \text{(From Eq.~\ref{eq:regret1})}\\ 
& = \dfrac{T}{K} \E_{\sigma_i\sim \rho_i({\hat{R}_1,\hat{R}_2,..,\hat{R}_{i-1}})} \E_{\hat{R}_i}SumLoss(\sigma_i,\hat{R}_i) \\
& \quad (\because \text{SumLoss is linear in both arguments } \text{ and $\hat{R}_i$ is unbiased)} \\
& = \dfrac{T}{K} \E_{\hat{R}_i} \E_{\sigma_i\sim\rho_i({\hat{R}_1,\hat{R}_2,..,\hat{R}_{i-1}})} SumLoss(\sigma_i,\hat{R}_i) .
\end{align*}
In the last step above, we crucially used the fact that, since random distribution $\rho_i$ is independent of $\hat{R}_i$, the order of expectations is interchangeable. Taking expectation w.r.t. $\hat{R}_1,\hat{R}_2,..,\hat{R}_{i-1}$, we get,
\begin{equation}
\begin{split}
\label{eq:intermediate}
\E_{\hat{R}_1,\ldots, \hat{R}_{i-1}} \E_{\sigma_t \sim \rho_i({\hat{R}_1,\hat{R}_2,..,\hat{R}_{i-1}})} & \sum_{t \in [B_i]} SumLoss(\sigma_t,R_t) \\
& = \dfrac{T}{K} \E_{\hat{R}_1,\ldots, \hat{R}_{i-1}, \hat{R}_i} E_{\sigma_i \sim \rho_i({\hat{R}_1,\hat{R}_2,..,\hat{R}_{i-1}})} SumLoss(\sigma_i,\hat{R}_i) .
\end{split}
\end{equation}

Thus, 
\begin{align*}
\E \sum_{t=1}^T SumLoss(\sigma_t,R_t) & = \E \sum_{i=1}^K \sum_{t \in [B_i]} SumLoss(\sigma_t,R_t)\\
& = \sum_{i=1}^K E_{\hat{R}_1,\ldots, \hat{R}_{i-1}} E_{\sigma_t \sim \rho_i({\hat{R}_1,\hat{R}_2,..,\hat{R}_{i-1}})}\sum_{t \in [B_i]} SumLoss(\sigma_t,R_t)\\
& = \dfrac{T}{K} \sum_{i=1}^K \E_{\hat{R}_1,\ldots, \hat{R}_{i-1}, \hat{R}_i} \E_{\sigma_i \sim \rho_i({\hat{R}_1,\hat{R}_2,..,\hat{R}_{i-1}})} SumLoss(\sigma_i,\hat{R}_i)\\
&  \text{(From Eq.~\ref{eq:intermediate})}\\
& = \dfrac{T}{K} \E_{\hat{R}_1,\ldots, \hat{R}_K}  \E_{\sigma_i \sim \rho_i({\hat{R}_1,\hat{R}_2,..,\hat{R}_{i-1}})} \sum_{i=1}^K SumLoss(\sigma_i,\hat{R}_i) 
\end{align*}

Now using Eq.~\ref{eq:regret2}, we can upper bound the last term above as
\begin{align*}
& \le \dfrac{T}{K}\{ \E_{\hat{R}_1,\ldots,\hat{R}_K}[\underset{\sigma}{\min} \sum_{i=1}^K \sigma ^{-1} \cdot \hat{R}_i ]+ C \sqrt{K}\} \\
& \le\dfrac{T}{K}\{ \underset{\sigma}{\min} \sum_{i=1}^K \sigma^{-1} \cdot R^{avg}_i + C \sqrt{K}\} \\
&\quad \text{(Jensen's \ Inequality)}\\
& \le \underset{\sigma}{\min} \sum_{t=1}^T \sigma^{-1} \cdot R_t + C \dfrac{T}{\sqrt{K}} \\
& = \underset{\sigma}{\min} \sum_{t=1}^T SumLoss(\sigma, R_t)+ C \dfrac{T}{\sqrt{K}}.
\end{align*}

{\bf Effect of exploration}: Since in each block $B_i$, $\ceil{m/k}$ rounds are reserved for exploration, where we do not draw $\sigma_t$ from distribution $\rho_i$, we need to account for it in our regret bound. Exploration leads to an extra regret of $C^{I} \ceil{m/k} K$, where $C^{I}$ is a constant depending on the loss under consideration and specific full information algorithm used. The extra regret is because loss in each of the exploration rounds is at most $C^{I}$ and there are a total of $\ceil{m/k} K$ exploration rounds over all $K$ blocks. Thus, overall regret :
\begin{align}
\E\left[ \sum_{t=1}^T SumLoss(\sigma_t,R_t) \right]- \underset{\sigma}{\min} \sum_{t=1}^T SumLoss(\sigma, R_t) \le C^{I} \ceil{m/k} K + C \dfrac{T}{\sqrt{K}} .
\end{align}
Now we optimize over $K$, to get:
\begin{equation}
\begin{aligned}
\label{eq:mainforregret}
\E\left[\sum_{t=1}^T SumLoss(\sigma_t,R_t)\right] \le   \underset{\sigma}{\min}\sum_{t=1}^T SumLoss(\sigma,R_t)  +  2 (C^I)^{1/3} C^{2/3} \ceil{m/k}^{1/3}T^{2/3}\end{aligned}
\end{equation}

\end{proof}

{\bf Proof of Corollary~\ref{efficientregret-SumLoss}}:

\begin{proof}
We only need to instantiate the constants $C$ and  $C^{I}$ from Theorem~\ref{regret}, with respect to SumLoss and FTPL.
FTPL has the following parameters in its regret bound, for any online full information linear optimization problem: $D$ is the $\ell_1$ diameter of learner's action set, $R$ is upper bound on difference between losses of $2$ actions on same information vector and $A$ is the $\ell_1$ diameter of the set of information vectors (adversary's action set).

For SumLoss, it can be easily calculated that $R = \sum_{i=1}^m \sigma^{-1}(i) R(i) = O(m^2)$,  $D = \sum_{i=1}^m \sigma^{-1}(i)= O(m^2)$,  and $A = \sum_{i=1}^m R(i)= O(m)$ .

FTPL gets $O(C \sqrt{T})$ regret over $T$ rounds when $\epsilon= \sqrt{\frac{D}{R A T}}$.  Here, $C= 2 \sqrt{DRA}$ and $C^{I}= R$. Substituting the values of $D,R, A$, we conclude. 
\end{proof}

{\bf Extension of results from SumLoss to DCG and Precision@n}: 

{\bf DCG}: Due to structural differences, there are minor differences in definitions and proofs of theorems for SumLoss and DCG. We give pointers in the in proving that local observability condition fails to hold for DCG, when restricted to top $1$ feedback. We can skip the explicit proof of global observability, since the application of Algorithm~\ref{alg:top-k} already establishes that $O(T^{2/3})$ regret can be achieved. 

With slight abuse of notations, the loss matrix $L$ implicitly means gain matrix, where entry in cell $\{i,j\}$ of $L$ is $f(\sigma_i) \cdot g(R_j)$. The columns of feedback matrix $H$ are expanded to account for greater number of moves available to adversary (due to multi-graded relevance vectors).  In Definition 1, learner action $i$ is optimal if $\ell_i \cdot p \ge \ell_j \cdot p, \ \forall j \neq i$.

In Definition 2, the maximum number of distinct elements that can be in a row of $H$ is $n+1$. The signal matrix now becomes $S_i \in \{0,1\}^{(n+1) \times 2^m}$, where $(S_i)_{j,\ell}= \mathbbm{1}(H_{i,\ell}= j-1)$. 

\emph{Local Observability Fails}: Since we are trying to establish a lower bound, it is sufficient to show it for binary relevance vectors, since the adversary can only be more powerful otherwise. 

In Lemma~\ref{pareto-optimal}, proved for SumLoss, $\ell_i \cdot p$ equates to $f(\sigma) \cdot \E [R]$. From definition of DCG, and from the structure and properties of $f(\cdot)$, it is clear that $\ell_i \cdot p$ is \emph{maximized} under the same condition, i.e, $\E[R(\sigma_i(1)] \ge \E[R(\sigma_i(2)] \ge \ldots \ge \E[R(\sigma_i(m)]$. Thus, all actions are Pareto-optimal.

Careful observation of Lemma~\ref{neighbor-actions} shows that it is directly applicable to DCG, in light of extension of Lemma~\ref{pareto-optimal} to DCG.

Finally, just like in SumLoss, simple calculations with $m=3$ and $n=1$, in light of Lemma~\ref{pareto-optimal} and \ref{neighbor-actions}, show that local observability condition fails to hold.

We show the calculations:
\begin{equation*}
\begin{aligned}
 S_{\sigma_1}= S_{\sigma_2}=
  \left[ {\begin{array}{cccccccc}
   1 & 1 & 1 & 1 & 0 & 0 & 0 & 0 \\
   0 & 0 & 0 & 0 & 1 & 1 & 1 & 1 \\
  \end{array} } \right]\\
\end{aligned}
\end{equation*}

\begin{equation*}
\begin{aligned}
\ell_{\sigma_1}= & [0, 1/2, 1/\log_2 3, 1/2 + 1/\log_2 3, 1, 3/2, \\
& 1 + 1/\log_2 3, 3/2 + 1/\log_2 3]\\
\ell_{\sigma_2}= & [0, 1/\log_2 3, 1/2, 1/2 + 1/\log_2 3, 1, 1 + 1/\log_2 3,\\
& 3/2, 3/2 + 1/\log_2 3]
\end{aligned}
\end{equation*}

It is clear that $\ell_1 - \ell_2 \notin Col(S^{\top}_{\sigma_1})$. Hence, Definition 5 fails to hold.

{\bf Proof of Corrolary~\ref{efficientregret-DCG}}

For DCG, the parameters of FTPL are: $R= \sum_{i=1}^m f^s(\sigma^{-1}(i))g^s(R(i))= O(m(2^n-1)),  D= \sum_{i=1}^m f^s(\sigma^{-1}(i))= O(m), A= \sum_{i=1}^m g^s(R(i))= O(m(2^n-1))$. Again, $C= 2 \sqrt{DRA}$ and $C^{I}= R$. \\

{\bf Precision@n}:

{\bf Proof of Corrolary~\ref{efficientregret-Precision@n}}

For Precision@$n$, the parameters of FTPL are: $D= \sum_{i=1}^m f^s(\sigma^{-1}(i))= O(n), R= \sum_{i=1}^m f^s(\sigma^{-1}(i))g^s(R(i))= O(n), A= \sum_{i=1}^m g^s(R(i))= O(m)$. Again, $C= 2 \sqrt{DRA}$ and $C^{I}= R$. \\

{\bf Non-existence of Sublinear Regret Bounds for NDCG, AP and AUC}

We show via simple calculations that for the case $m=3$, global observability condition fails to hold for NDCG, when feedback is restricted to top ranked item, and relevance vectors are restricted to take binary values. It should be noted that allowing for multi-graded relevance vectors only makes the adversary more powerful; hence proving for binary relevance vectors is enough.

The intuition behind failure to satisfy global observability condition is that the $NDCG(\sigma,R)$ = $f(\sigma) \cdot g(R)$, where where $g(r)= R/ Z(R)$ (see Sec.\ref{rankingmeasures-NC} ). Thus, $g(\cdot)$ cannot be represented by univariate, scalar valued functions. This makes it impossible to write the difference between two rows of the loss matrix as linear combination of columns of (transposed) signal matrices.

Similar intuitions hold for AP and AUC.

{\bf Proof of Lemma \ref{globalfailsfornormalized}}

\begin{proof}
We will first consider NDCG and then, AP and AUC.

{\bf NDCG}:

The first and last row of Table \ref{lossmatrix-table}, when calculated for NDCG, are:
\begin{equation*}
\begin{aligned}
\ell_{\sigma_1}=  [1, 1/2, 1/\log_2 3, (1+\log_2 3/2))/(1+\log_2 3), 1, 3/(2(1+1/\log_2 3)), 1, 1]\\
\ell_{\sigma_6}=  [1, 1, \log_2 2/\log_2 3, 1, 1/2, 3/(2(1+1/\log_2 3)), (1+(\log_2 3)/2))/(1+\log_2 3), 1]
\end{aligned}
\end{equation*}

We remind once again that NDCG is a gain function, as opposed to SumLoss. 

The difference between the two vectors is:
\begin{equation*}
\begin{aligned}
\ell_{\sigma_1} - \ell_{\sigma_6}= [0, -1/2, 0, -\log_2 3/(2(1+\log_2 3)), 1/2, 0, \log_2 3/(2(1+\log_2 3)), 0].
\end{aligned}
\end{equation*} 
 
The signal matrices are same as SumLoss:

\begin{equation*}
\begin{aligned}
 S_{\sigma_1}= S_{\sigma_2}=
  \left[ {\begin{array}{cccccccc}
   1 & 1 & 1 & 1 & 0 & 0 & 0 & 0 \\
   0 & 0 & 0 & 0 & 1 & 1 & 1 & 1 \\
  \end{array} } \right]\\
\\
S_{\sigma_3}= S_{\sigma_5}=
  \left[ {\begin{array}{cccccccc}
   1 & 1 & 0 & 0 & 1 & 1 & 0 & 0 \\
   0 & 0 & 1 & 1 & 0 & 0 & 1 & 1 \\
  \end{array} } \right]\\
\\
S_{\sigma_4}= S_{\sigma_6}=
  \left[ {\begin{array}{cccccccc}
   1 & 0 & 1 & 0 & 1 & 0 & 1 & 0 \\
   0 & 1 & 0 & 1 & 0 & 1 & 0 & 1 \\
  \end{array} } \right]\\
\end{aligned}
\end{equation*}

It can now be easily checked that $\ell_{\sigma_1} - \ell_{\sigma_6}$ does not lie in the (combined) column span of the (transposed) signal matrices.

We show similar calculations for AP and AUC.

{\bf AP}:

We once again take $m=3$. The first and last row of Table \ref{lossmatrix-table}, when calculated for AP, is:
\begin{equation*}
\begin{aligned}
&\ell_{\sigma_1}= [1, 1/3, 1/2, 7/12, 1, 5/6, 1, 1]\\
&\ell_{\sigma_6}= [1, 1, 1/2, 1, 1/3, 5/6, 7/12, 1]
\end{aligned}
\end{equation*}

Like NDCG, AP is also a gain function.

The difference between the two vectors is:
\begin{equation*}
\ell_{\sigma_1} - \ell_{\sigma_6}= [0, -2/3, 0, -5/12, 2/3, 0, 5/12, 0].
\end{equation*} 
 
The signal matrices are same as in the SumLoss case:
\begin{equation*}
\begin{aligned}
 S_{\sigma_1}= S_{\sigma_2}=
  \left[ {\begin{array}{cccccccc}
   1 & 1 & 1 & 1 & 0 & 0 & 0 & 0 \\
   0 & 0 & 0 & 0 & 1 & 1 & 1 & 1 \\
  \end{array} } \right]\\
\\
S_{\sigma_3}= S_{\sigma_5}=
  \left[ {\begin{array}{cccccccc}
   1 & 1 & 0 & 0 & 1 & 1 & 0 & 0 \\
   0 & 0 & 1 & 1 & 0 & 0 & 1 & 1 \\
  \end{array} } \right]\\
\\
S_{\sigma_4}= S_{\sigma_6}=
  \left[ {\begin{array}{cccccccc}
   1 & 0 & 1 & 0 & 1 & 0 & 1 & 0 \\
   0 & 1 & 0 & 1 & 0 & 1 & 0 & 1 \\
  \end{array} } \right]\\
\end{aligned}
\end{equation*}

It can now be easily checked that $\ell_{\sigma_1} - \ell_{\sigma_6}$ does not lie in the (combined) column span of the (transposed) signal matrices.
\end{proof}

{\bf AUC}:

For AUC, we will show the calculations for $m=4$. This is because global observability does hold with $m=3$, as the normalizing factors for all relevance vectors with non-trivial mixture of $0$ and $1$ are same (i.e, when relevance vector has 1 irrelevant and 2 relevant objects, and 1 relevant and 2 irrelevant objects, the normalizing factors are same).  The normalizing factor changes from $m=4$ onwards; hence global observability fails.

Table~\ref{lossmatrix-table} will be extended since $m=4$. Instead of illustrating the full table, we point out the important facts about the loss matrix table with $m=4$ for AUC.

The $2^4$ relevance vectors heading the columns are:

$R_1= 0000,\ R_2= 0001,\ R_3=0010,\ R_4=0100,\ R_5=1000,\ R_6=0011,\ R_7=0101,\ R_8=1001,\ R_9=0110,\ R_{10}=1010,\ R_{11}=1100,\ R_{12}=0111,\ R_{13}=1011,\ R_{14}=1101,\ R_{15}=1110,\ R_{16}=1111$.

We will calculate the losses  of 1st and last (24th) action, where $\sigma_1= 1234$ and $\sigma_{24}=4321$.

\begin{equation*}
\begin{aligned}
&\ell_{\sigma_1}= [0, 1, 2/3, 1/3, 0, 1, 3/4, 1/2, 1/2, 1/4, 0, 1, 2/3, 1/3, 0, 0]\\
&\ell_{\sigma_{24}}= [0, 0, 1/3, 2/3, 1, 0, 1/4, 1/2, 1/2, 3/4, 1, 0, 1/3, 2/3, 1, 0]
\end{aligned}
\end{equation*}

AUC, like SumLoss, is a loss function.

The difference between the two vectors is:
\begin{equation*}
\begin{aligned}
\ell_{\sigma_1} - \ell_{\sigma_{24}}=  [0, 1, 1/3, -1/3, -1, 1, 1/2, 0, 0, -1/2, -1, 1, 1/3, -1/3, -1, 0].
\end{aligned}
\end{equation*} 

The signal matrices for AUC with $m=4$ will be slightly different. This is because there are 24 signal matrices, corresponding to 24 actions. However, groups of 6 actions will share the same signal matrix. For example, all 6 permutations that place object 1 first will have same signal matrix, all 6 permutations that place object 2 first will have same signal matrix, and so on. For simplicity, we denote the signal matrices as $S_1, S_2, S_3, S_4$, where $S_i$ corresponds to signal matrix where object $i$ is placed at top. We have:

\begin{equation*}
\begin{aligned}
 S_1=
  \left[ {\begin{array}{cccccccccccccccc}
   1 & 1 & 1 & 1 & 0 & 1 & 1 & 0 & 1 & 0 & 0 & 1 & 0 & 0 & 0 & 0 \\
   0 & 0 & 0 & 0 & 1 & 0 & 0 & 1 & 0 & 1 & 1 & 0 & 1 & 1 & 1 & 1 \\
  \end{array} } \right]\\
\\
 S_2=
  \left[ {\begin{array}{cccccccccccccccc}
   1 & 1 & 1 & 0 & 1 & 1 & 0 & 1 & 0 & 1 & 0 & 0 & 1 & 0 & 0 & 0 \\
   0 & 0 & 0 & 1 & 0 & 0 & 1 & 0 & 1 & 0 & 1 & 1 & 0 & 1 & 1 & 1 \\
  \end{array} } \right]\\
\\
 S_3=
  \left[ {\begin{array}{cccccccccccccccc}
   1 & 1 & 0 & 1 & 1 & 0 & 1 & 1 & 0 & 0 & 1 & 0 & 0 & 1 & 0 & 0 \\
   0 & 0 & 1 & 0 & 0 & 1 & 0 & 0 & 1 & 1 & 0 & 1 & 1 & 0 & 1 & 1 \\
  \end{array} } \right]\\
\\
 S_4=
  \left[ {\begin{array}{cccccccccccccccc}
   1 & 0 & 1 & 1 & 1 & 0 & 0 & 0 & 1 & 1 & 1 & 0 & 0 & 0 & 1 & 0 \\
   0 & 1 & 0 & 0 & 0 & 1 & 1 & 1 & 0 & 0 & 0 & 1 & 1 & 1 & 0 & 1 \\
  \end{array} } \right]\\
\end{aligned}
\end{equation*}

It can now be easily checked that $\ell_{\sigma_1} - \ell_{\sigma_{24}}$ does not lie in the (combined) column span of transposes of $S_{1}, S_{2}, S_{3}, S_4$.

%% file: Appendix-B.tex
{\bf Proof of Lemma~\ref{unbiasedestimator}}: We restate the lemma before giving the proof, for ease of reading:

{\bf Lemma~\ref{unbiasedestimator}}: Let $F: \mathbb{R}^m \mapsto \mathbb{R}^a$ be a vector valued function, where $m\ge 1$, $a\ge 1$. For a fixed $x \in \mathbb{R}^m$, let $k$ entries of $x$ be observed at random. That is, for a fixed probability distribution $\mathbb{P}$ and some random $\sigma \sim \mathbb{P}(S_m)$, observed tuple is $\{\sigma, x_{\sigma(1)}, \ldots, x_{\sigma(k)}\}$. The necessary condition for existence of an unbiased estimator of $F(x)$, that can be constructed from $\{\sigma, x_{\sigma(1)}, \ldots, x_{\sigma(k)}\}$, is that it should be  possible to decompose $F(x)$ over $k$ (or less) coordinates of $x$ at a time. That is, $F(x)$ should have the following structure:
\begin{equation*}
F(x)= \sum\limits_{(i_1,i_2,\ldots,i_{\ell}) \in \ \perm{m}{\ell}} h_{i_1,i_2,\ldots,i_{\ell}}(x_{i_1}, x_{i_2},\ldots, x_{i_{\ell}}) 
\end{equation*}
where $\ell \le k$, $\perm{m}{\ell}$ is $\ell$ permutations of $m$ and $h: \mathbb{R}^{\ell} \mapsto \mathbb{R}^a$.
Moreover, when $F(x)$ can be written in form of Eq~\ref{eq:decoupling} , with $\ell=k$, an unbiased estimator of $F(x)$, based on $\{\sigma, x_{\sigma(1)}, \ldots, x_{\sigma(k)}\}$, is, 
\begin{equation*}
g(\sigma, x_{\sigma(1)}, \ldots, x_{\sigma(k)})=  \dfrac{\sum\limits_{(j_1,j_2,\ldots, j_k) \in S_k}h_{\sigma(j_1),\ldots, \sigma(j_k)}(x_{\sigma(j_1)}, \ldots, x_{\sigma(j_k)})}{\sum\limits_{\substack{(j_1,\ldots, j_k) \in S_k}} p(\sigma(j_1),\ldots,\sigma(j_k))} 
\end{equation*}
where $S_k$ is the set of $k!$ permutations of $[$k$]$ and $p(\sigma(1),\ldots, \sigma(k))$ is as in Eq~\ref{eq:shortprob} .

\begin{proof}
For a fixed $x \in \mathbb{R}^m$ and probability distribution $\mathbb{P}$, let the random permutation be $\sigma \sim \mathbb{P}(S_m)$ and the observed tuple be $\{\sigma, x_{\sigma(1)}, \ldots, x_{\sigma(k)}\}$. Let $\hat{G}= G(\sigma, x_{\sigma(1)}, \ldots, x_{\sigma(k)})$ be an unbiased estimator of $F(x)$ based on the random observed tuple. Taking expectation, we get:
\begin{equation*}
\begin{aligned}
\begin{split}
F(x)= &\E_{\sigma \sim \mathbb{P}} \left[\hat{G} \right] = \sum_{\pi \in S_m} \mathbb{P}(\pi) G(\pi, x_{\pi(1)}, \ldots, x_{\pi(k)}) \\
&= \sum_{(i_1, i_2, \ldots, i_k) \in \ \perm{m}{k}} \sum_{\pi \in S_m} \mathbb{P}(\pi) \mathbbm{1}(\pi(1)=i_1, \pi(2)=i_2, \ldots, \pi(k)= i_k) G(\pi, x_{i_1}, x_{i_2}, \ldots, x_{i_k}) 
\end{split} 
\end{aligned}
\end{equation*}
We note that $\mathbb{P}(\pi) \in [0,1]$ is independent of $x$ for all $\pi \in S_m$. Then we can use the following construction of function $h(\cdot)$:
\begin{equation*}
h_{i_1,i_2,\ldots, i_k} (x_{i_1}, \ldots, x_{i_k})=  \sum_{\pi \in S_m} \mathbb{P}(\pi) \mathbbm{1}(\pi(1)=i_1, \pi(2)=i_2, \ldots, \pi(k)= i_k) G(\pi, x_{i_1}, x_{i_2}, \ldots, x_{i_k}) 
\end{equation*}
 and thus,
\begin{equation*}
F(x)= \sum\limits_{(i_1,i_2,\ldots,i_{k}) \in \ \perm{m}{k}} h_{i_1,i_2,\ldots,i_{k}}(x_{i_1}, x_{i_2},\ldots, x_{i_{}}) 
\end{equation*}
Hence, we conclude that for existence of an unbiased estimator based on the random observed tuple, it should be  possible to decompose $F(x)$ over $k$ (or less) coordinates of $x$ at a time. The ``less than $k$'' coordinates arguement follows simply by noting that if $F(x)$ can be decomposed over $\ell$ coordinates at a time ($\ell <k$) and observation tuple is \{$\sigma, x_{\sigma(1)}, \ldots, x_{\sigma(k)})$\}, then any $k - \ell$ observations can be thrown away and the rest used for construction of the unbiased estimator.

The construction of the unbiased estimator proceeds as follows:

Let $F(x)= \sum_{i=1}^m h_i(x_i)$ and feedback is for top-1 item ($k=1$). The  unbiased estimator according to Lemma.~\ref{unbiasedestimator} is:

\begin{equation*}
g(\sigma,x_{\sigma(1)})= \dfrac{h_{\sigma(1)}(x_{\sigma(1)})}{p(\sigma(1))} = \dfrac{h_{\sigma(1)}(x_{\sigma(1)})}{\sum_{\pi} \mathbb{P}(\pi) \mathbbm{1}(\pi (1)=\sigma(1))}
\end{equation*}

Taking expectation w.r.t. $\sigma$, we get:
\begin{equation*}
\E_{\sigma}[g(\sigma,x_{\sigma(1)})]= \sum_{i=1}^m \dfrac{h_{i}(x_{i}) (\sum_{\pi} \mathbb{P}(\pi) \mathbbm{1}(\pi(1)=i))}{\sum_{\pi} \mathbb{P}(\pi) \mathbbm{1}(\pi(1)=i)}= \sum_{i=1}^m h_{i}(x_i)= F(x)
\end{equation*}

Now, let $F(x)= \sum\limits_{i \neq j =1}^m h_{i,j}(x_{i},x_{j})$ and the feedback is for top-2 item ($k=2$). The  unbiased estimator according to Lemma.~\ref{unbiasedestimator} is:
\begin{equation*}
\begin{split}
g(\sigma,x_{\sigma(1)}, x_{\sigma(2)})& = \dfrac{h_{\sigma(1), \sigma(2)}(x_{\sigma(1)}, x_{\sigma(2)}) + h_{\sigma(2),\sigma(1)}(x_{\sigma(2)}, x_{\sigma(1)})}{p(\sigma(1), \sigma(2))+ p(\sigma(2), \sigma(1))}\\
\end{split}
\end{equation*}

We will use the fact that for any 2 permutations $\sigma_1, \sigma_2$, which places the same 2 objects in top-2 positions but in opposite order, estimators based on $\sigma_1$ (i.e,  $g(\sigma_1, x_{\sigma_1(1)}, x_{\sigma_1(2)})$) and $\sigma_2$ (i.e, $g(\sigma_2, x_{\sigma_2(1)}, x_{\sigma_2(2)})$) have same numerator and denominator.  For eg., let $\sigma_1(1)=i, \sigma_1(2)=j$. Numerator and denominator for $g(\sigma_1, x_{\sigma_1(1)}, x_{\sigma_1(2)})$  are $h_{i,j}(x_i,x_j) + h_{j,i}(x_j,x_i)$ and $p(i,j)+ p(j,i)$ respectively. Now let $\sigma_2(1)=j, \sigma_2(2)=i$. Then numerator and denominator for $g(\sigma_2, x_{\sigma_2(1)}, x_{\sigma_2(2)})$ are $ h_{j,i}(x_j,x_i)+ h_{i,j}(x_i,x_j)$  and $p(j,i)+ p(i,j)$ respectively.

Then, taking expectation w.r.t. $\sigma$, we get:
\begin{equation*}
\begin{split}
\E_{\sigma}{g(\sigma,x_{\sigma(1)}, x_{\sigma(2)})} &= \sum_{i \neq j =1}^m \dfrac{(h_{i,j}(x_i,x_j) + h_{j,i}(x_j,x_i))p(i,j)}{p(i,j)+p(j,i)} \\
& = \sum_{i> j =1}^m \dfrac{(h_{i,j}(x_i,x_j) + h_{j,i}(x_j,x_i)) (p(i,j)+p(j,i))}{p(i,j)+p(j,i)} \\
&= \sum_{i>j=1}^m (h_{i,j}(x_i,x_j) + h_{j,i}(x_j,x_i)) = \sum_{i \neq j =1}^m h_{i,j}(x_i,x_j)  = F(x)
\end{split}
\end{equation*}

This chain of logic can be extended for any $k \ge 3$. Explicitly, for general $k \le m$, let $\mathbb{S}(i_1,i_2,\ldots, i_k)$ denote all permutations of the set $\{i_1, \ldots, i_k \}$. Then, taking expectation of the unbiased estimator will give:

\begin{equation*}
\begin{split}
&\E_{\sigma}{g(\sigma,x_{\sigma(1)}, \ldots, x_{\sigma(k)})} \\
& = \sum_{(i_1,i_2,\ldots,i_k) \in \ \perm{m}{k}} \dfrac{\left(\sum\limits_{(j_1,\ldots,j_k) \in \mathbb{S}(i_1,\ldots,i_k)}h_{j_1,\ldots,j_k}(x_{j_1},\ldots,x_{j_k})\right)p(i_1,\ldots,i_k)}{\sum\limits_{(j_1,\ldots,j_k) \in \mathbb{S}(i_1,\ldots,i_k)} p(j_1,\ldots,j_k)} \\
& = \sum_{i_1 > i_2 > \ldots > i_k =1}^m \dfrac{\left(\sum\limits_{(j_1,\ldots,j_k) \in \mathbb{S}(i_1,\ldots,i_k)}h_{j_1,\ldots,j_k}(x_{j_1},\ldots,x_{j_k})\right)\left(\sum\limits_{(j_1,\ldots,j_k) \in \mathbb{S}(i_1,\ldots,i_k)} p(j_1,\ldots,j_k)\right)}{\sum\limits_{(j_1,\ldots,j_k) \in \mathbb{S}(i_1,\ldots,i_k)} p(j_1,\ldots,j_k)} \\
&= \sum_{i_1 > i_2 > \ldots > i_k =1}^m \left(\sum\limits_{(j_1,\ldots,j_k) \in \mathbb{S}(i_1,\ldots,i_k)}h_{j_1,\ldots,j_k}(x_{j_1},\ldots,x_{j_k}) \right)\\
& = \sum\limits_{(i_1, i_2, \ldots, i_{k}) \in \ \perm{m}{k}} h_{i_1,i_2,\ldots,i_{k}}(x_{i_1}, x_{i_2},\ldots, x_{i_k})  = F(x)
\end{split}
\end{equation*}

{\bf Note}: For $k=m$, i.e., when the full feedback is received, the unbiased estimator is:

\begin{equation*}
\begin{split}
g(\sigma, x_{\sigma(1)}, \ldots, x_{\sigma(m)}) & =  \dfrac{\sum\limits_{(j_1,j_2,\ldots, j_m) \in S_m}h_{\sigma(j_1),\ldots, \sigma(j_m)}(x_{\sigma(j_1)}, \ldots, x_{\sigma(j_m)})}{\sum\limits_{\substack{(j_1,\ldots, j_m) \in S_m}} p(\sigma(j_1),\ldots,\sigma(j_m))} \\
&=  \dfrac{\sum\limits_{(i_1, i_2, \ldots, i_m) \in \ \perm{m}{m}} h_{i_1,\ldots, i_m}(x_{i_1}, \ldots, x_{i_m})}{1}= F(x)
\end{split}
\end{equation*}

Hence, with full information, the unbiased estimator of $F(x)$ is actually $F(x)$ itself, which is consistent with the theory of unbiased estimator.

\end{proof}

{\bf Proof of Lemma~\ref{expectednorm}} :

\begin{proof}
All our unbiased estimators are of the form $X^{\top} f(s,R,\sigma)$. We will actually get a bound on $f(s,R,\sigma)$ by using Lemma~\ref{normexpr} and $p \to q$ norm relation, to equate out $X$:
\begin{equation*}
\begin{split}
\|\tilde{z}\|_2 & = \|X^{\top} f(s,R,\sigma)\|_2 \le \|X^{\top}\|_{1 \to 2} \|f(s,R,\sigma)\|_1 \le  R_D   \|f(s,R,\sigma)\|_1
\end{split}
\end{equation*} 
since $R_D \ge \max_{j=1}^m \| X_{j:} \|_2$.

{\bf Squared Loss}: The unbiased estimator of gradient of squared loss, as given in the main text, is:
\begin{equation*}
\tilde{z}= X^{\top} (2 (s  - \dfrac{R(\sigma(1)) e_{\sigma(1)}}{p(\sigma(1))})) 
\end{equation*} 
where $p(\sigma(1))= \sum_{\pi \in S_m} \mathbb{P} (\pi) \mathbbm{1}(\pi(1)= \sigma(1))$ ($\mathbb{P}= \mathbb{P}_t$ is the distribution at round $t$ as in Algorithm~\ref{alg:RTop-kF} )

Now we have:
\begin{equation*}
\|s  - \dfrac{R(\sigma(1)) e_{\sigma(1)}}{p(\sigma(1))}\|_1 \le m R_D U  + \dfrac{R_{max}}{p(\sigma(1))} \le \dfrac{m R_D U R_{max}}{p(\sigma(1)}
\end{equation*}
Thus, taking expectation w.r.t $\sigma$, we get:
\begin{equation*}
\E_{\sigma} \|\tilde{z}\|^2_2 \le m^2 R_D^4 U^2 R_{max}^2 \E_{\sigma} {\dfrac{1}{p(\sigma(1))^2}}= m^2 R_D^4 U^2 R_{max}^2 \sum_{i=1}^m \dfrac{p(i)}{p^2(i)} 
\end{equation*}
Now, since $p(i) \ge \dfrac{\gamma}{m}$, $\forall \ i$, we get: $\E_{\sigma} \|\tilde{z}\|^2_2 \le $ $ \dfrac{C^{sq}}{\gamma}$, where $C^{sq}= m^4 R_D^4 U^2 R_{max}^2$. 

{\bf RankSVM Surrogate}:  The unbiased estimator of gradient of the RankSVM surrogate, as given in the main text, is:
\begin{equation*}
\tilde{z}=  X^{\top} \left(\dfrac{h_{s,\sigma(1),\sigma(2)}(R(\sigma(1)), R(\sigma(2))) + h_{s,\sigma(2),\sigma(1)}(R(\sigma(2)), R(\sigma(1)))}{p(\sigma(1),\sigma(2))+ p(\sigma(2),\sigma(1))}\right)
\end{equation*} 
where $h_{s, i,j}(R(i),R(j))= \mathbbm{1}(R(i) >R(j)) \mathbbm{1}(1+s(j)>s(i)) (e_j - e_i)$ and $p(\sigma(1), \sigma(2))= \sum\limits_{\pi \in S_m} \mathbb{P} (\pi) \mathbbm{1}(\pi(1)= \sigma(1), \pi(2)= \sigma(2))$ ($\mathbb{P}= \mathbb{P}_t$ as in  Algorithm~\ref{alg:RTop-kF})).

Now we have:
\begin{equation*}
\|\dfrac{h_{s,\sigma(1),\sigma(2)}(R_{\sigma(1)}, R_{\sigma(2)}) + h_{s,\sigma(2),\sigma(1)}(R_{\sigma(2)}, R_{\sigma(1)})}{p(\sigma(1),\sigma(2))+ p(\sigma(2),\sigma(1))}\|_1 \le \dfrac{2 }{ p(\sigma(1),\sigma(2)) + p(\sigma(2),\sigma(1))}
\end{equation*}

Thus, taking expectation w.r.t $\sigma$, we get:
\begin{equation*}
\E_{\sigma} \|\tilde{z}\|^2_2 \le 4 R_D^2 \E_{\sigma} \dfrac{1}{(p(\sigma(1), \sigma(2)) + p(\sigma(2),\sigma(1)))^2} \le 4 R_D^2 \sum_{i>j}^m \dfrac{p(i,j) +p(j,i)}{(p(i,j) + p(j,i))^2}
\end{equation*}
Now, since $p(i, j) \ge \dfrac{\gamma}{m^2}$, $\forall \ i,j$, we get: $\E_{\sigma} \|\tilde{z}\|^2_2 \le $ $ \dfrac{C^{svm}}{\gamma}$, where $C^{svm}= O(m^4 R_D^2) $. 

{\bf KL based Surrogate}:  The unbiased estimator of gradient of the KL based surrogate, as given in the main text, is:
\begin{equation*}
\tilde{z}= X^{\top} \left( \dfrac{(\exp(s(\sigma(1))) - \exp(R(\sigma(1))))e_{\sigma(1)}}{p(\sigma(1))}\right)
\end{equation*}
where $p(\sigma(1))= \sum_{\pi \in S_m} \mathbb{P} (\pi) \mathbbm{1}(\pi(1)= \sigma(1))$ ($\mathbb{P}= \mathbb{P}_t $ as in Alg.~\ref{alg:RTop-kF}) ).

Now we have:
\begin{equation*}
\|  \dfrac{(\exp(s(\sigma(1))) - \exp(R(\sigma(1))))e_{\sigma(1)}}{p(\sigma(1))}\|_1 \le \dfrac{\exp(R_D U)}{p(\sigma(1))}
\end{equation*}

Thus, taking expectation w.r.t $\sigma$, we get:
\begin{equation*}
\E_{\sigma} \|\tilde{z}\|^2_2 \le R_D^2 \exp(2R_D U) \E_{\sigma} \dfrac{1}{p(\sigma(1))^2} 
\end{equation*}
Following the same arguement as in squared loss, we get: $\E_{\sigma} \|\tilde{z}\|^2_2 \le $ $ \dfrac{C^{KL}}{\gamma}$, where $C^{KL}= m^2 R_D^2 \exp(2R_D U)$. 

\end{proof}

{\bf Proof of Lemma~\ref{RankSVM}} :

\begin{proof}
Let $m=3$. Collection of all terms which are functions of 1st coordinate of $R$, i.e, $R(1)$, in the gradient of RankSVM  is: $\mathbbm{1}(R(1)>R(2))\mathbbm{1}(1+s(2)>s(1)) (e_2-e_1)$ + $\mathbbm{1}(R(2)>R(1))\mathbbm{1}(1+s(1)>s(2)) (e_1-e_2)$ + $\mathbbm{1}(R(1)>R(3))\mathbbm{1}(1+s(3)>s(1)) (e_3-e_1)$ + $\mathbbm{1}(R(3)>R(1))\mathbbm{1}(1+s(1)>s(3)) (e_1-e_3)$.  Now let $s(1)=1, s(2)=0, s(3)=0$. Then the collection becomes: $\mathbbm{1}(R(2)>R(1))(e_1-e_2)$ + $\mathbbm{1}(R(3)>R(1))(e_1-e_3)$ = $(\mathbbm{1}(R(2)>R(1)) + \mathbbm{1}(R(3)>R(1)))e_1 - \mathbbm{1}(R(2)>R(1))e_2 - \mathbbm{1}(R(3)>R(1))e_3$. Now, if the gradient can be decomposed over each coordinate of $R$, then the collection of terms associated with $R(1)$ should \emph{only and only be a function} of $R(1)$. Specifically, $(\mathbbm{1}(R(2)>R(1)) + \mathbbm{1}(R(3)>R(1)))$ (the non-zero coefficient of $e_1$) should be a function of only $R(1)$ (similarly for $e_2$ and $e_3$). 

Now assume that the $(\mathbbm{1}(R(2)>R(1)) + \mathbbm{1}(R(3)>R(1)))$ can be expressed as a function of $R(1)$ only. Then the difference between the coefficient's values, for the following two cases: $R(1)=0, R(2)=0, R(3)=0$ and $R(1)=1, R(2)=0,R(3)=0$, would be same as the difference between the coefficient's values, for the following two cases: $R(1)=0, R(2)=1,R(3)=1$ and $R(1)=1,R(2)=1,R(3)=1$ (Since the difference would be affected only by change in $R(1)$ value). It can be clearly seen that the change in value between the first two cases is: $0-0=0$, while the change in value between the second two cases is: $2-0=2$. Thus, we reach a contradiction. 
\end{proof}

{\bf Proof of Lemma~\ref{listnet}} :

\begin{proof}
The term associated with the 1st coordinate of $R$, i.e, $R(1)$, in the gradient of ListNet is =  $ \sum_{i=1}^m \left(\dfrac{-\exp(R(i))}{\sum_{j=1}^m \exp(R(j))} + \dfrac{\exp(s(i))}{\sum_{j=1}^m \exp(s(j))} \right) e_i$ (in fact, the same term is associated with every coordinate of $R$).

Specifically, $f(R)= \left(\dfrac{-\exp(R(1))}{\sum_{j=1}^m \exp(R(j))} + \dfrac{\exp(s_1)}{\sum_{j=1}^m \exp(s(j))} \right)$ is the non-zero coefficient of $e_1$, associated with $R(1)$. Now, if $f(R)$ would have only been a function of $R(1)$, then $\dfrac{\partial^2 f(R)}{\partial R(1) \partial R(j)}$, $\forall \ j \neq 1$ would have been zero. It can be clearly seen this is not the case.

Now, the term associated jointly with $R(1)$ and $R(2)$, in the gradient of ListNet is same as before, i.e, $ \sum_{i=1}^m \left(\dfrac{-\exp(R(i))}{\sum_{j=1}^m \exp(R(j))} + \dfrac{\exp(s(i))}{\sum_{j=1}^m \exp(s(j))} \right) e_i$  (since $R(1)$ and $R(2)$ are present in all the summation terms of the gradient).

Specifically, $f(R)= \left(\dfrac{-\exp(R(i))}{\sum_{j=1}^m \exp(R(j))} + \dfrac{\exp(s(i))}{\sum_{j=1}^m \exp(s(j))} \right)$ is the non-zero coefficient of $e_1$. Now, if $f(R)$ would have only been a function of $R(1)$ and $R(2)$, then $\dfrac{\partial^3 f(R)}{\partial R(1) \partial R(2) \partial R(j)}$, $\forall j \neq 1, j \neq 2$ would have been zero.  It can be clearly seen this is not the case.

The same argument can be extended for any $k <m$. 

\end{proof}

{\bf Proof of Theorem.~\ref{impossiblegame}}:

\begin{proof} ({\it Sketch})
The proof builds on the proof of hopeless finite action partial monitoring games given by \cite{piccolboni2001discrete}. An examination of their proof of Theorem.\ 3 indicates that for hopeless games, there have to exist two probability distributions (over adversary's actions), which are indistinguishable in terms of feedback but the optimal learner's actions for the distributions are different. We first provide a mathematical explanation as to why such existence lead to hopeless games. Then, we provide a characterization of indistinguishable probability distributions in our problem setting, and then exploit the characterization of optimal actions for NDCG calibrated surrogates (Theorem~\ref{calibration}) to explicitly construct two such probability distributions. This proves the result. Full proof is given below.
\end{proof}

\begin{proof}
We will first fix the setting of the online game. We consider $m=3$ and fixed the document matrix $X \in \mathbb{R}^{3 \times 3}$ to be the identity. At each round of the game, the adversary generates the fixed $X$ and the learner chooses a score vector $s \in \reals^3$. Making the matrix $X$ identity makes the distinction between weight vectors $w$ and scores $s$ irrelevant since $s = Xw = w$. We note that allowing the adversary to vary $X$ over the rounds only makes him more powerful, which can only increase the regret. We also restrict the adversary to choose binary relevance vectors. Once again, allowing adversary to choose multi-graded relevance vectors only makes it more powerful. Thus, in this setting, the adversary can now choose among $2^3=8$ possible relevance vectors.  The learner's action set is infinite, i.e., the learner can choose any score vector $s= Xw= \reals^m$. The loss function $\phi(s,R)$ is any NDCG calibrated surrogate and feedback is the relevance of top-ranked item at each round, where ranking is induced by sorted order (descending) of score vector. We will use $p$ to denote randomized adversary one-short strategies, i.e. distributions over the $8$ possible relevance score vectors. Let $s^*_p = \argmin_{s} \E_{R \sim p} \phi(s,R)$. We note that in the definition of NDCG calibrated surrogates, \cite{ravikumar2011ndcg} assume that the optimal score vector for each distribution over relevance vectors is unique and we subscribe to that assumption. The assumption was taken to avoid some boundary conditions.

It remains to specify the choice of $U$, a bound on the Euclidean norm of the weight vectors (same as score vectors for us right now) that is used to define the best loss in hindsight. It never makes sense for the learner to play anything outside the set $\cup_p s^*_p$ so that we can set $U = \max \{ \|s\|_2 \::\: s \in \cup_p s^*_p \}$.

The paragraph following Lemma 6 of Thm. 3 in \cite{piccolboni2001discrete} gives the main intuition behind the argument the authors developed to prove hopelessness of finite action partial monitoring games. To make our proof self contained, we will explain the intuition in a rigorous way. 

{\bf Key insight}: Two adversary strategies $p,\tilde{p}$ are said to be indistinguishable from the learner's feedback perspective, if for every action of the learner, the probability distribution over the feedbacks received by learner is the same for $p$ and $\tilde{p}$. Now assume that adversary always selects actions according to one of the two such indistinguishable strategies. Thus, the learner will always play one of $s^*_p$ and $s^*_{\tilde{p}}$. By uniqueness, $s^*_p \neq s^*_{\tilde{p}}$. Then, the learner incurs a constant (non-zero) regret on any round where adversary plays according to $p$ and learner plays $s^*_{\tilde{p}}$, or if the adversary plays according to $\tilde{p}$ and learner plays $s^*_p$.   We show that in such a setting, adversary can simply play according to $(p+\tilde{p})/2$ and the learner suffers an expected regret of $\Omega(T)$.

Assume that the adversary selects $\{R_1,\ldots, R_T\}$ from product distribution $\otimes p$. Let the number of times the learner plays $s^*_p$ and $s^*_{\tilde{p}}$ be denoted by random variables $N^{p}_1$ and $N^{p}_2$ respectively, where $N^{p}$ shows the exclusive dependence on $p$. It is always true that $N^p_1 + N^{p}_2= T$. Moreover, let the expected per round regret be $\epsilon_p$ when learner plays $s^*_{\tilde{p}}$ , where the expectation is taken over the randomization of adversary. Now, assume that adversary selects $\{R_1,\ldots, R_T\}$ from product distribution $\otimes \tilde{p}$. The corresponding notations become $N^{\tilde{p}}_1$ and $N^{\tilde{p}}_2$ and $\epsilon_{\tilde{p}}$. Then, 
\[
\E_{(R_1,\ldots,R_T) \sim \otimes p} \E_{(s_1,\ldots,s_T)} [\text{Regret}((s_1,\ldots,s_T),(R_1,\ldots,R_T))] = 0 \cdot \E [N^{p}_1] + \epsilon_{p} \cdot \E [N^{p}_2]
\]
 and 
 \[
 \E_{(R_1,\ldots,R_T) \sim \otimes \tilde{p}} \E_{(s_1,\ldots,s_T)} [\text{Regret}((s_1,\ldots,s_T),(R_1,\ldots,R_T))] = \epsilon_{\tilde{p}} \cdot \E [N^{\tilde{p}}_1] + 0 \cdot \E [N^{\tilde{p}}_2]
\]
  Since $p$ and $\tilde{p}$ are indistinguishable from perspective of learner, $\E [N^{p}_1]= \E[N^{\tilde{p}}_1]= \E [N_1]$ and $\E [N^{p}_2]= \E[N^{\tilde{p}}_2]= \E [N_2]$. That is, the random variable denoting number of times $s^*_p$ is played by learner does not depend on adversary distribution (same for $s^*_{\tilde{p}}$.). Using this fact and averaging the two expectations, we get: 
 \begin{equation*}
 \begin{aligned}
 \begin{split}
 \E_{(R_1,\ldots,R_T) \sim {\frac{\otimes p+\otimes \tilde{p}}{2}}} \E_{(s_1,\ldots,s_T)} [\text{Regret}((s_1,\ldots,s_T),(R_1,\ldots,R_T))] & = \frac{\epsilon_{\tilde{p}}}{2} \cdot \E [N_1] + \frac{\epsilon_{p}}{2}  \cdot\E [N_2] \\
 & \ge \min (\frac{\epsilon_p}{2},\frac{\epsilon_{\tilde{p}}}{2})  \cdot \E [N_1 + N_2]= \epsilon \cdot T
 \end{split}
 \end{aligned}
\end{equation*}

Since 
\[
\begin{aligned}
\sup_{R_1,\ldots,R_T} \E [\text{Regret}((s_1,\ldots,s_T)& ,(R_1,\ldots,R_T))]  \ge \\
& \E_{(R_1,\ldots,R_T) \sim {\frac{\otimes p+ \otimes \tilde{p}}{2}}} \E_{(s_1,\ldots,s_T)} [\text{Regret}((s_1,\ldots,s_T),(R_1,\ldots,R_T))]
\end{aligned}
\]
we conclude that for every learner algorithm, adversary has a strategy, s.t. learner suffers an expected regret of $\Omega(T)$.

Now, the thing left to be shown is the existence of two indistinguishable distributions $p$ and $\tilde{p}$, s.t. $s^*_p \neq s^*_{\tilde{p}}$.

{\bf Characterization of indistinguishable strategies in our problem setting}: Two adversary's strategies $p$ and $\tilde{p}$ will be indistinguishable, in our problem setting, if for every score vector $s$, the relevances of the top-ranked item, according to s, are same for relevance vector drawn from $p$ and $\tilde{p}$.  Since relevance vectors are restricted to be binary, mathematically, it means that $\forall s$, $\mathbb{P}_{R \sim p} (R(\pi_s(1))=1)= \mathbb{P}_{R \sim \tilde{p}} (R(\pi_s(1))=1)$ (actually, we also need $\forall s$, $\mathbb{P}_{R \sim p} (R(\pi_s(1))=0)= \mathbb{P}_{R \sim \tilde{p}} (R(\pi_s(1))=0)$, but due to the binary nature, $\mathbb{P}_{R \sim p} (R(\pi_s(1))=1)= \mathbb{P}_{R \sim \tilde{p}} (R(\pi_s(1))=1)$ $\implies$ $\mathbb{P}_{R \sim p} (R(\pi_s(1))=0)= \mathbb{P}_{R \sim \tilde{p}} (R(\pi_s(1))=0)$). Since the equality has to hold $\forall s$, this implies $\forall j \in [m]$, $\mathbb{P}_{R \sim p} (R(j)=1)= \mathbb{P}_{R \sim \tilde{p}} (R(j)=1)$ (as every item will be ranked at top by some score vector). Hence, $\forall j \in [m]$, $\E_{R \sim p} [R(j)]= \E_{R \sim \tilde{p}} [R(j)]$ $\implies$ $\E_{R \sim p} [R]= \E_{R \sim \tilde{p}} [R]$. It can be seen clearly that the chain of implications can be reversed. Hence, $\forall s$, $\mathbb{P}_{R \sim p} (R(\pi_s(1))=1)= \mathbb{P}_{R \sim \tilde{p}} (R(\pi_s(1))=1)$ $\Longleftrightarrow$ $\E_{R \sim p} [R]= \E_{R \sim \tilde{p}} [R]$.

{\bf Explicit adversary strategies}: Following from the discussion so far and Theorem~\ref{calibration}, if we can show existence of two strategies $p$ and $\tilde{p}$ s.t. $\E_{R \sim p} [R]= \E_{R \sim \tilde{p}} [R]$, but $\argsort \left(\E_{R \sim p} \left[\frac{G({\bf R})}{Z(R)}\right]\right) \neq \argsort \left(\E_{R \sim \tilde{p}} \left[\frac{G({\bf R})}{Z(R)}\right]\right)$, we are done.

The 8 possible relevance vectors (adversary's actions) are $(R_1,R_2,R_3,R_4,R_5,R_6,R_7,R_8)= (000, 110, 101, 011, 100, 010, 001, 111)$. Let the two probability vectors be: 
\[
\begin{aligned}
& p= (0.0, 0.1, 0.15, 0.05, 0.2, 0.3,0.2,0.0) \\
& \tilde{p}= (0.0,0.3,0.0,0.0,0.15,0.15,0.4,0.0).
\end{aligned}
\]

The data is provided in table format in Table.~\ref{relevanceprobability-table}.  

Under the two distributions, it can be checked that $\E_{R \sim p} [R]= \E_{R \sim \tilde{p}} [R]= (0.45,0.45,0.4)^{\top}$. 

However, $\E_{R \sim p} \left[\frac{G({\bf R})}{Z(R)}\right]= (0.3533,0.3920,0.3226)^{\top}$, but $\E_{R \sim \tilde{p}} \left[\frac{G({\bf R})}{Z(R)}\right]= (0.3339, 0.3339, 0.4000)^{\top}$. Hence, $\argsort \left(\E_{R \sim p} \left[\frac{G({\bf R})}{Z(R)}\right]\right)= [2,1,3]^{\top}$ but $\argsort \left(\E_{R \sim \tilde{p}} \left[\frac{G({\bf R})}{Z(R)}\right]\right) \in \{[3,1,2]^{\top}, [3,2,1]^{\top}\}$.

\begin{table}[t] 
\caption{Relevance and probability vectors.} 
\label{relevanceprobability-table}
\begin{center}
\tabcolsep=0.110cm
\begin{tabular}{c c c c c c c c c}  
\hline 
\hline $p$ & 0.0 & 0.1 & 0.15 & 0.05 & 0.2 & 0.3 & 0.2 & 0.0\\ [0.4ex]
\hline $\tilde{p}$ & 0.0 & 0.3 & 0.0 & 0.0 & 0.15 & 0.15 & 0.4 & 0.0\\ [0.4ex] 
\hline Rel.& $R_1$ & $R_2$ & $R_3$ & $R_4$ & $R_5$ & $R_6$ & $R_7$ & $R_8$ \\ [0.4ex] 
\hline 
& 0 & 1 & 1 & 0 & 1 & 0 & 0 & 1 \\ 
& 0 & 1 & 0 & 1 & 0 & 1 & 0 & 1 \\ 
& 0 & 0 & 1 & 1 & 0 & 0 & 1 & 1  \\ 
\hline 
\end{tabular} 
\end{center}
\end{table}

\end{proof}